\documentclass[11pt]{article}

\usepackage[margin=1in]{geometry}

\usepackage{microtype}
\usepackage{times}

\usepackage{amsmath}
\usepackage{amssymb}
\usepackage{mathtools}
\usepackage{amsthm}
\usepackage{bm}
\usepackage{graphicx}
\usepackage{booktabs}
\usepackage{siunitx}
\usepackage{caption}
\usepackage{subcaption}
\usepackage{dcolumn}

\usepackage{natbib}
\usepackage{hyperref}
\usepackage[capitalize,noabbrev]{cleveref}

\usepackage{algorithm}
\usepackage{algorithmic}

\usepackage[textsize=tiny]{todonotes}

\newcommand{\NN}{\textsc{NN}}
\newcommand{\prop}{\textsc{LRKME}}
\newcommand{\soft}{\textsc{SoftImpute}}

\DeclareMathOperator*{\argmin}{argmin}

\newcommand{\tp}{\intercal}

\newcommand{\bigO}{\ensuremath{\mathop{}\mathopen{}\mathcal{O}\mathopen{}}}
\newcommand{\smallO}{\scalebox{0.7}{$\mathcal{O}$}}

\newcommand{\ind}{\perp\!\!\!\!\perp}

\newcommand{\EE}{{\mathbb{E}}}

\newcommand{\calE}{{\mathcal{E}}}

\newcommand{\calH}{{\mathcal{H}}}

\newcommand{\calK}{{\mathcal{K}}}

\newcommand{\calS}{{\mathcal{S}}}
\newcommand{\calT}{{\mathcal{T}}}
\newcommand{\calU}{{\mathcal{U}}}

\newcommand{\calX}{{\mathcal{X}}}

\makeatletter
\newcommand{\shorteq}{\mathrel{\mkern0.2mu\mathpalette\shorteq@\relax\mkern0.2mu}}
\newcommand{\shorteq@}[2]{\scalebox{0.7}[1]{$\m@th#1=$}}

\newcommand{\shortminus}{\mathrel{\mkern0.2mu\mathpalette\shortminus@\relax\mkern0.2mu}}
\newcommand{\shortminus@}[2]{\scalebox{0.7}[1]{$\m@th#1-$}}

\newcommand{\longeq}[1]{\mathrel{\mathpalette\longeq@{#1}}}
\newcommand{\longeq@}[2]{%
  \begingroup
  \sbox\z@{$\m@th#1=$}%
  \ifdim#2<\wd\z@
    \resizebox{#2}{\height}{\box\z@}%
  \else
    \ifdim#2<3\wd\z@
      \hbox to #2{$\m@th#1=\hss=\hss=\hss=$}%
    \else
      \hbox to #2{$\m@th#1=\cleaders\hbox to 0.2\wd\z@{\hss$#1=$\hss}\hfil=$}%
    \fi
  \fi
  \endgroup
}
\makeatother

\DeclareEmphSequence{\bfseries\itshape} 

\theoremstyle{plain}
\newtheorem{theorem}{Theorem}[section]

\newtheorem{lemma}[theorem]{Lemma}

\theoremstyle{definition}
\newtheorem{definition}[theorem]{Definition}
\newtheorem{assumption}[theorem]{Assumption}
\theoremstyle{remark}
\newtheorem{remark}[theorem]{Remark}

\title{Low-rank Distributional Matrix Completion}

\author{
Jiayi Wang \\
\textit{University of Texas at Dallas} \\[8pt]
Raymond K.W. Wong \\
\textit{Texas A\&M University}
}

\date{}

\begin{document}
\maketitle

\begin{abstract}
We study a distributional generalization of the matrix completion problem in which each entry of the target matrix is a probability distribution rather than a scalar. In this setting, only a subset of matrix entries is observed, and even for observed entries, the underlying distributions are not directly accessible; instead, we observe finitely many samples drawn from them. To represent distributional entries, we employ kernel mean embeddings and introduce a notion of Tucker rank for distribution-valued matrices to capture their low-rank structure. The infinite-dimensional nature of kernel embeddings poses significant methodological challenges. To address this, we introduce functional unfolding operators that link the proposed distributional low-rank structure to the classical Tucker rank for finite-dimensional tensors. Based on this framework, we propose a novel estimator for distributional matrix completion. We establish non-asymptotic error bounds that characterize the statistical performance of the estimator. Extensive experiments on synthetic data and a real-world application demonstrate the effectiveness of the proposed method.
\end{abstract}
\section{Introduction}
\label{sec:intro}
Low-rank matrix completion addresses the problem of recovering missing entries from a partially observed and potentially contaminated data matrix, where the low-rank assumption serves as a crucial structural regularizer to facilitate identification. Due to its flexibility and effectiveness, low-rank matrix completion has been widely applied in a broad range of domains, including recommendation systems \citep{kang2016top,gurini2018temporal,chen2021kernel,zhang2021artificial}, signal processing \citep{weng2012low,zhang2020low,chen2023high,yuchi2023bayesian}, image recovery \citep{changjun2012single,cao2014image,zheng2024scale} and data imputation \citep{chen2017ensemble,chen2020nonconvex,xu2023hrst}. A substantial body of work has been devoted to the theoretical and algorithmic development of matrix completion.
However, most existing approaches are tailored to scalar-valued matrices and do not directly extend to target matrices with non-scalar entries. In this work, we study matrix completion problems in which the entries are probability distributions.


As a motivating example, consider modeling the distribution of daily taxi trip counts between pairs of origin and destination locations (see Section~\ref{sec:real}) in order to capture both the volume and variability of traffic flows. Another example arises in the analysis of quarterly earnings estimates for public companies, where the goal is to study the distribution of forecasts issued by different banks and traders. In this setting, companies index the rows and fiscal quarters across multiple years index the columns; see Section~1.1 of \citet{feitelberg2024distributional} for a detailed discussion.

For observed entries, the underlying distributions can be inferred from the corresponding samples. In practice, however, some entries may be entirely unobserved. In the taxi-trip example, missingness may result from incomplete trip records or limited GPS coverage in certain locations. In the earnings-estimates example, firms may release information irregularly, leading to missing entries for particular quarters or time periods. These challenges highlight the need for efficient methods to impute missing distributional entries.


To the best of our knowledge, there is only one existing work that explicitly addresses distributional matrix completion. In \citet{feitelberg2024distributional}, the authors propose to impute missing entries using a nearest-neighbor-based approach. Specifically, for a missing entry located at $(i,j)$, similar rows are identified by using averaged empirical $2$-Wasserstein distances over the overlapping columns with available observations. The missing distribution is then estimated by taking the empirical Wasserstein barycenter of the distributions associated with these neighboring rows in column $j$.
Their method relies on a simplified observation setting in which column $j$ is observed for all selected neighbors, and each neighbor has the same number of observations $N$ in column $j$. Moreover, the final estimate is represented by $N$ synthetic samples generated from the estimated distribution, which requires $N$ to be sufficiently large in order to accurately recover the underlying distribution. In addition, due to the computational challenges associated with evaluating $2$-Wasserstein distances and computing Wasserstein barycenters for multivariate distributions, the proposed algorithm does not readily extend to multidimensional distributional settings.

In contrast, our approach accommodates a more general observation pattern, allowing the number of samples to vary across observed entries, and naturally extends to multidimensional distributional data.
We represent distributional entries using kernel mean embeddings (KMEs), thereby embedding probability distributions into a Hilbert space and circumventing explicit manifold constraints.
Unlike the nearest-neighbor strategy in \citet{feitelberg2024distributional}, our method is formulated within an empirical risk minimization framework and enforces a global low-rank structure on the distributional matrix through carefully designed regularization.

The proposed low-rank structure is motivated by Tucker decomposition \cite{tucker1966some} in tensor analysis, but operates in a more challenging setting in which one mode is infinite-dimensional. To tackle this, we introduce unfolding operators that map the distributional matrix into tensor products of two spaces (see Definition~\ref{def:unfolding}), and impose nuclear-norm regularization on the corresponding induced linear operators to promote low-rankness. This modeling strategy simultaneously encourages intrinsic dimension reduction in the functional component of the kernel mean embeddings and induces a low-rank structure in the associated expectation functional matrices (see definition in \eqref{eqn:expectation_mat}).
Although the infinite-dimensional nature of the distributional representations poses optimization challenges, we show that, with appropriate regularization and by leveraging the reproducing property of the underlying kernel, the solution admits a finite-dimensional representation. As a result, the proposed optimization problem can be solved efficiently using convex algorithms. On the theoretical side, we establish error bounds for the aggregate squared maximum mean discrepancy (MMD) across all matrix entries. We show that the resulting rate matches that of classical nuclear-norm-regularized matrix completion for (approximately) low-rank matrices, despite the additional technical challenges arising from replacing scalar entries with functional representations. Notably, unlike \citet{feitelberg2024distributional}, our consistency guarantees do not necessarily require the number of samples for each observed entry to diverge.

\section{Problem Setup}
\label{sec:setup}

Let $\mathcal{X}\subseteq\mathbb{R}^d$, and denote by $\mathcal{P}(\mathcal{X})$
the collection of probability distributions of random variables taking values in $\mathcal{X}$.
We consider a distribution-valued matrix $\rho^*\in\mathcal{P}(\mathcal{X})^{m_1\times m_2}$, whose $(i,j)$-th entry $\rho^*(i,j)$ is a probability distribution on $\mathcal{X}$.
The goal of this paper is to recover $\rho^*$ from partial and indirect observations.
Specifically, observations are available only for a subset of the entries of $\rho^*$, and even for these observed entries, the underlying distribution is not directly observed.
Instead, for each observed entry $(i,j)$, we collect i.i.d.\ samples
$X^{(ij)}_1,\dots,X^{(ij)}_{N_{ij}} \sim \rho^*(i,j)$,
where $N_{ij}$ denotes the sample size for entry $(i,j)$.
Let $T(i,j)\in\{0,1\}$ be an indicator of whether the distribution $\rho^*(i,j)$ is observed: $T(i,j)=1$ if samples are collected from $\rho^*(i,j)$, and $T(i,j)=0$ otherwise.
The observed data thus consist of
$\{X^{(ij)}_1,\dots,X^{(ij)}_{N_{ij}} : T(i,j)=1, i=1,\dots,m_1, j=1,\dots,m_2\}$.
Throughout the paper, we assume that the collections
$\{(T(i,j),X^{(ij)}_1,\dots,X^{(ij)}_{N_{ij}}): i=1,\dots,m_1, j=1,\dots,m_2\}$
are mutually independent.

There are various common ways to represent a distribution on $\mathcal{X}$.
For examples, probability mass/density functions and cumulative distribution function.
However, many of these classical functional representations exhibits manifold constraints.
To avoid manifold constraints,
kernel mean embedding (KME) has become a popular tool for Hilbertian embedding 
of probability distributions.
It allows distributions to be treated as elements of a reproducing kernel Hilbert space (RKHS)
\citep[e.g.][]{smola2007hilbert, muandet2017kernel}.  
Moreover, KME supports efficient empirical estimation from samples, and admits a natural metric—the Maximum Mean Discrepancy (MMD)—that facilitates distribution comparison and learning in a wide range of machine learning tasks \citep{gretton2012kernel}. More specifically, the MMD is defined as 
\begin{align*}
&\mathrm{MMD}(\rho_1,\rho_2;\mathcal{H})\\
&:=
\sup_{\substack{\|f\|_{\mathcal{H}} \le 1}}
\left(
\mathbb{E}_{X \sim \rho_1}[f(X)]
-
\mathbb{E}_{Y \sim \rho_2}[f(Y)]
\right)\\
&= \|  \mathbb{E}_{X \sim \rho_1} \, K(X,\cdot) -  \mathbb{E}_{X \sim \rho_2} \, K(X,\cdot) \|_\calH,
\end{align*}
where \(K : \mathcal{X} \times \mathcal{X} \to \mathbb{R}\) is a reproducing kernel and \(\mathcal{H}\) denotes the associated RKHS.
Throughout this work, we represent the entries of the target distribution-valued matrix
\(
\rho^* \in \mathcal{P}(\mathcal{X})^{m_1 \times m_2}
\)
by their corresponding KMEs.
Specifically, for each \((i,j)\), we define
\begin{equation} 
    \mu^*(i,j,\cdot)
    \;=\;
    \mathbb{E}_{X \sim \rho^*(i,j)} \, K(X,\cdot)
    \;\in\;
    \mathcal{H}.   
\end{equation}
When the kernel \(K\) is {characteristic} (for example, Gaussian and Laplacian kernels), the KME is injective, implying that the underlying distribution \(\rho^*(i,j)\) is uniquely determined by \(\mu^*(i,j,\cdot)\) \citep{sriperumbudur2010hilbert}.
Consequently, recovering the distribution-valued matrix \(\rho^*\) is equivalent to recovering its KME matrix \(\mu^*\).
In the following, we therefore focus on the problem of completing the KME matrix
\(
\mu^* \in \mathcal{H}^{m_1 \times m_2}. 
\)

\section{Low Rank Modeling}
Following the classical literature on noisy matrix completion \citep[e.g.,][]{klopp,candes2010matrix}, we aim to impose a suitable ``low-rank'' structural assumption on $\mu^*$ to facilitate recovery from partial observations. 
We begin by observing that the KME matrix \(\mu^*\) can be naturally viewed as a three-way tensor taking values in the tensor product space
\(
\mathbb{R}^{m_1} \otimes \mathbb{R}^{m_2} \otimes \mathcal{H}.
\)
Our goal is to impose low-rank structure on \(\mu^*\) in the Tucker sense.
Note that,
unlike standard tensor learning settings,
the last space $\mathcal{H}$ is often infinite dimensional,
resulting in a partly infinite-dimensional tensor with finite marginal dimensions for the first two modes and an infinite marginal dimension for the last mode.
To deal with this type of tensors,
we need more involved mathematical construct for the investigation of multilinear rank.
Motivated by the notion of functional unfolding operators \citep{wang2022low}, we define three unfolding operators that map \(\mu^*\) into tensor products of two spaces, thereby inducing linear operators whose ranks can be characterized.

Notice that the tensor product space $\mathbb{R}^{m_1}\times \mathbb{R}^{m_2}\times \mathcal{H}$ is the completion of the linear span of all elementary tensors
of the form
\(a_1 \otimes a_2 \otimes h\)
under the inner product
$\langle a_1 \otimes a_2 \otimes h, a'_1 \otimes a'_2 \otimes h' \rangle_{\mathbb{R}^{m_1}\times \mathbb{R}^{m_2}\times \mathcal{H}} = (a_1^\tp a_1') (a_2^\tp a_2') \langle h, h' \rangle_\calH$,
where \(a_1, a_1' \in \mathbb{R}^{m_1}\), \(a_2, a_2' \in \mathbb{R}^{m_2}\), and \(h, h' \in \mathcal{H}\).
\begin{definition} 
\label{def:unfolding}
The unfolding operators \(\mathcal{U}_j\), \(j=1,2,3\), for any elementary tensor   $a_1 \otimes a_2 \otimes h$, are defined as follows:
\begin{align*}
\mathcal{U}_1 &: \mathbb{R}^{m_1} \otimes \mathbb{R}^{m_2} \otimes \mathcal{H}
\to
\mathbb{R}^{m_1} \otimes (\mathbb{R}^{m_2} \otimes \mathcal{H}), \\
&\mathcal{U}_1(a_1 \otimes a_2 \otimes h)
=
a_1 \otimes (a_2 \otimes h); \\[0.5em]
\mathcal{U}_2 &: \mathbb{R}^{m_1} \otimes \mathbb{R}^{m_2} \otimes \mathcal{H}
\to
\mathbb{R}^{m_2} \otimes (\mathbb{R}^{m_1} \otimes \mathcal{H}), \\
&\mathcal{U}_2(a_1 \otimes a_2 \otimes h)
=
a_2 \otimes (a_1 \otimes h); \\[0.5em]
\mathcal{U}_3 &: \mathbb{R}^{m_1} \otimes \mathbb{R}^{m_2} \otimes \mathcal{H}
\to
\mathcal{H} \otimes (\mathbb{R}^{m_1} \otimes \mathbb{R}^{m_2}), \\
&\mathcal{U}_3(a_1 \otimes a_2 \otimes h)
=
h \otimes (a_1 \otimes a_2).
\end{align*}
These definitions extend to any element $\mu \in \mathbb{R}^{m_1} \otimes \mathbb{R}^{m_2} \otimes \calH$ by linearity.
\end{definition}
For a finite-dimensional tensor, the unfolding operators defined above are equivalent to the standard matricization of the tensor defined in Definition \ref{def:matricize} in Appendix \ref{sec:add_tensor}. 
Any element $f \in \mathcal{F}_1 \otimes \mathcal{F}_2$ can be naturally identified with a linear operator mapping $\mathcal{F}_2$ to $\mathcal{F}_1$, defined by
\(
g \in \mathcal{F}_2 \;\longmapsto\; \left\langle f(\star, \cdot),\, g(\cdot) \right\rangle_{\mathcal{F}_2}.
\)
With a slight abuse of notation, we use $f$ to denote both the tensor element and its induced linear operator.
Under this operator interpretation, an element
\(
\mathcal{U}_1 \mu \in \mathbb{R}^{m_1} \otimes \bigl(\mathbb{R}^{m_2} \otimes \mathcal{H}\bigr)
\)
can be viewed as a linear operator mapping the space $\mathbb{R}^{m_2} \otimes \mathcal{H}$ into $\mathbb{R}^{m_1}$.
 Similar operator interpretations apply to $\mathcal{U}_2 \mu$ and $\mathcal{U}_3 \mu$.
And we define the {multilinear rank} of \(\mu \in \mathbb{R}^{m_1} \otimes \mathbb{R}^{m_2} \otimes \calH \) as
\[
\bigl(
\mathrm{rank}(\mathcal{U}_1 \mu),
\mathrm{rank}(\mathcal{U}_2 \mu),
\mathrm{rank}(\mathcal{U}_3 \mu)
\bigr),
\]
where the rank is understood as the rank of the corresponding induced linear operator. 
When the Hilbert space \(\mathcal{H}\) is finite-dimensional, this definition coincides with the classical Tucker rank of a finite-dimensional third-order tensor \citep{tucker1966some,kolda2009tensor}.

Next, we illustrate how a small multilinear rank of $\mu^\ast$ leads to meaningful structural understanding. Assume that $\mathcal{H}$ is a separable Hilbert space. 
Let \(\{e_k\}_{k=1}^q\) denote an orthonormal basis of the Hilbert space \(\mathcal{H}\),
where \(q\) can be infinite.
Then, for each \((i,j)\), the KME admits the expansion
\begin{align}
    \mu^*(i,j,\cdot)
    \;=\;
    \sum_{k=1}^q B(i,j,k)\, e_k(\cdot),
\end{align}
where \(B(i,j,k) \in \mathbb{R}\), \(k = 1,\dots,q\), are the corresponding expansion coefficients.

Suppose that \(\mu^*\) has multilinear rank \((r_1, r_2, r_3)\), with \(r_1 < m_1\), \(r_2 < m_2\), and \(r_3 < q\).
Then there exist a tensor \(G \in \mathbb{R}^{r_1 \times r_2 \times r_3}\) and factor matrices
\(U_1 \in \mathbb{R}^{m_1 \times r_1}\),
\(U_2 \in \mathbb{R}^{m_2 \times r_2}\), and
\(U_3 \in \mathbb{R}^{q \times r_3}\),
such that the coefficient tensor \(B\) admits the Tucker decomposition \citep{tucker1966some}
\begin{align}
\label{eqn:tucker_decomp}
    B
    \;=\;
    G \times_1 U_1 \times_2 U_2 \times_3 U_3,
\end{align}
where \(\times_j\) denotes the mode-\(j\) tensor–matrix product, defined in Definition~\ref{def:n-mode} in Appendix~\ref{sec:add_tensor}.
Here, \(G\) is referred to as the {core tensor}, and the matrices \(U_j\), \(j=1,2,3\), are often assumed to have orthonormal columns.
An illustration of the Tucker decomposition is provided in Figure~\ref{tucker_figure} in Appendix~\ref{sec:add_tensor}.

With such decomposition \eqref{eqn:tucker_decomp}, we have
\begin{align}
\label{eqn:lr-rep}
   &\mu^*(i,j, \cdot)\nonumber\\
   & = \sum_{l_1=1}^{r_1}\sum_{l_2=1}^{r_2}\sum_{l_3=1}^{r_3} G(l_1, l_2, l_3) U_{1}(i,l_1) U_2(j, l_2)s_{l_3}(\cdot), 
\end{align}
where $s_{l_3}(\cdot) = \sum_{k=1}^{q} U_3(k, l_3) e_k(\cdot)$, $k=1,\dots, r_3$.
This shows that only \(r_3\) basis functions \(\{s_{l_3}\}_{l_3=1}^{r_3}\) are sufficient to characterize the functional component of \(\mu^*\).

By the property of KMEs, for any \(f \in \mathcal{H}\),
\(
\mathbb{E}_{X \sim \rho(i,j)}[f(X)]
=
\langle f, \mu^*(i,j,\cdot) \rangle_{\mathcal{H}}.
\)
 Define the matrix \(A \in \mathbb{R}^{m_1 \times m_2}\) by
 \begin{align}
 \label{eqn:expectation_mat}
     A_f(i,j) := \mathbb{E}_{X \sim \rho(i,j)}[f(X)], \qquad f \in \calH.
 \end{align}
For example, if $f(t) = t$, $t \in \calX$ (i.e., an identity function), then $A$ is the mean matrix.  Substituting \eqref{eqn:lr-rep} into the above expression yields
\begin{align}
    &A_f(i,j)\nonumber\\
    &=
    \sum_{l_1=1}^{r_1}
    \sum_{l_2=1}^{r_2}
    \sum_{l_3=1}^{r_3}
    G(l_1,l_2,l_3)
    U_1(i,l_1)
    U_2(j,l_2)
    \langle f, s_{l_3} \rangle_{\mathcal{H}} \nonumber \\
    &=
    \sum_{l_1=1}^{r_1}
    \sum_{l_2=1}^{r_2}
    U_1(i,l_1)
    U_2(j,l_2)
    C_f(l_1,l_2),
\end{align}
where
\(
C_f(l_1,l_2)
:=
\sum_{l_3=1}^{r_3}
G(l_1,l_2,l_3)
\langle f, s_{l_3} \rangle_{\mathcal{H}}.
\)
Consequently, the matrix \(A_f\) admits the factorization
\(
A_f = U_1 C_f U_2^\tp, C_f \in \mathbb{R}^{r_1 \times r_2},
\)
and therefore has rank at most \(\min\{r_1, r_2\}\). By imposing a small multilinear rank on \(\mu^*\), we promote intrinsic dimension reduction in the functional component of the KME, thereby facilitating downstream analysis of the associated distributions, and at the same time induce a low-rank structure
for
the expectation functional matrix \(A_f\).


\section{Estimation}
\label{sec:estimation}
We adopt a regularized $M$-estimation approach to construct the proposed estimator \(\hat{\mu}\).
Specifically,
\begin{align} & \hat{\mu} = \argmin_{\mu \in \mathbb{R}^{m_1} \otimes \mathbb{R}^{m_2} \otimes \calH \; : \; \sup_{i,j}\|\mu(i,j,\cdot) \|_\calH \leq \alpha}\nonumber \\ & \Bigg[\frac{1}{m_1 m_2} \sum_{(i,j):T(i,j)=1} \Big\| \mu(i,j,\cdot) - \frac{1}{N_{ij}} \sum_{l=1}^{N_{ij}} K(X^{(ij)}_l,\cdot)\Big\|^2_\calH \label{eqn:loss}
\\ & \qquad + \lambda \Psi(\mu)\Bigg], \label{eqn: reg} \end{align}
where the term in \eqref{eqn:loss} involves the squared maximum mean discrepancies (MMD) between the estimated KME and the empirical KME at the observed entries.
The functional \(\Psi(\cdot)\) is a regularization term designed to promote low-rank structure in \(\mu\) while simultaneously ensuring the validity of the representer theorem (see Theorem~\ref{thm:representor}).
Here, \(\lambda > 0\) is a tuning parameter controlling the strength of regularization, and \(\alpha\) is an upper bound on the RKHS norm of each kernel mean embedding \(\mu(i,j,\cdot)\).
In the implementation, we simply take $\alpha$ to be large value (e.g., $\max_{i,j} 10\|\sum_{l=1}^{N_{ij}} K(X^{(ij)}_l,\cdot)/{N_{ij}} \|_\calH$).

For the regularization term \(\Psi(\mu)\), we adopt a nuclear-norm-based penalty as a convex surrogate to promote low-rank structure.
Specifically, we define
\begin{align}
\label{eqn:reg_nuclear}
    \Psi(\mu) = \sum_{j=1}^3 \beta_j \| \calU_j \mu \|_*, 
\end{align}
where \(\|\cdot\|_*\) denotes the nuclear norm of the induced operator corresponding to each unfolding.
The weights \(\beta_j > 0\) (such that \(\sum_{j=1}^3 \beta_j = 1\)) control the relative degree of low-rank regularization imposed on the different unfoldings of \(\mu\).
By default, we set $\beta_j=1/3$ for $j=1,2,3$.

Since $\hat{\mu} \in \mathbb{R}^{m_1} \otimes \mathbb{R}^{m_2} \otimes \mathcal{H}$, where $\mathcal{H}$ may be infinite-dimensional, the optimization problem in \eqref{eqn: reg} appears challenging due to its infinite-dimensional nature. However, by the following representer theorem, one can verify that the solution to \eqref{eqn: reg} admits a finite-dimensional representation. 
\begin{theorem}[Representer Theorem]
\label{thm:representor}
For any $(i,j)$, $\hat{\mu}(i,j,\dots)$ lies in the following finite-dimensional space 
\begin{align}
\label{eqn:sol-space}
    \calK = \mathrm{span} \{K(X_l^{(ij)}, \cdot), l = 1,\dots, N_{ij}, T_{ij} = 1\}.
\end{align}
\end{theorem}

Specifically,
the target optimization \eqref{eqn: reg} can be recast as a finite-dimensional optimization over parameter
$A \in \mathbb{R}^{m_1\times m_2 \times \Tilde{N}}$
such that
${\mu}(i,j,\cdot) = \sum_{l=1}^{\Tilde{N}} A(i,j,l) K(\bm x_l,\cdot)$, where 
$\Tilde{N} = \sum_{ij}T_{ij}N_{ij}$
and
$\bm x = [ x^{(i,j)}]_{(i,j):T_{i,j}=1}^\tp \in \mathbb{R}^{\Tilde{N}}$ with $x^{(i,j)} = [X^{(i,j)}_1,\dots, X^{(i,j)}_{N_{ij}}]^\tp \in \mathbb{R}^{N_{ij}}$.  It is easy to verify that 
\begin{multline*}
    \Big\| \mu(i,j,\cdot) - \frac{1}{N_{ij}} \sum_{l=1}^{N_{ij}} K(X^{(ij)}_l,\cdot)\Big\|^2_\calH \\
    = \sum_{l,l'=1,\dots,\Tilde{N}}A(i,j,l) A(i,j,l') K(\bm x_l, \bm x_{l'}) 
    - 2 \sum_{l}^{\Tilde{N}}\sum_{l'=1}^{N_{ij}} A(i,j,l)/N_{ij}K(\bm x_l, X_{l'}^{(ij)}) \\
    + \frac{1}{N^2_{ij}}\sum_{l,l'=1,\dots,N_{ij}} K(\bm  X_{l}^{(ij)}, X_{l'}^{(ij)}).
\end{multline*}
Re-expressing this term in vector-matrix notation,
one can show that the optimization \eqref{eqn:loss} involves the Gram matrix $K_G: = [K(\bm x_l, \bm x_m)]_{l,m=1,\dots,\Tilde{N}} \in \mathbb{R}^{\Tilde{N} \times \Tilde{N}}$.

In practice, however, the total number of observations \(\sum_{i,j} T_{ij} N_{ij}\) can be very large, rendering the optimization computationally expensive.
To address this issue, one may employ standard kernel approximation techniques such as the Nystr\"{o}m method \citep{williams2001using,drineas2005nystrom} or Random Fourier Features \citep{rahimi2007random} to obtain scalable approximations of the Gram matrix \(K_G\).  More specifically, we approximate 
$K_G$ by $Z Z^\tp$, where $Z \in \mathbb{R}^{(\sum_{i,j}T_{ij} N_{ij}) \times q}$, with $q \ll \sum_{i,j}T_{ij} N_{ij}$. 

Next, we rewrite the optimization into the form of (finite-dimensional) tensor optimization problem explicitly.
To this end, we construct the tensor \(Y \in \mathbb{R}^{m_1 \times m_2 \times q}\) as follows.
For each observed entry \((i,j)\) with \(T_{ij}=1\), let \(Y(i,j,\cdot)\in\mathbb{R}^q\) be the average of the rows of \(Z\) corresponding to the sample vector \(x^{(i,j)}\);
for unobserved entries with \(T_{ij}=0\), set \(Y(i,j,\cdot)=\bm{0}\).
With this construction, the optimization problem in \eqref{eqn: reg} admits the following re-parameterization:
\begin{align} \label{eqn:repar} &\hat{B}  = \argmin_{B \in \mathbb{R}^{m_1\times m_2 \times q}, \sup_{i,j}\|B(i,j,\cdot)\|_2 \leq \alpha} 
\left[ \sum_{ij} T_{ij} \|B(i,j,\cdot) - Y(i,j,\cdot)\|_2^2 + \lambda \left(\sum_{j=1}^3 \beta_j\| \calU_j B\|_*\right)\right]
\end{align}
where $B(i,j,\cdot), Y(i,j,\cdot)\in \mathbb{R}^q$ and $\|\cdot\|_2$ denotes the Euclidean norm.  
With $\hat{B}$, the estimator \(\hat{\mu}\) is obtained by
\[
\hat{\mu}(i,j,t)
\;=\;
z_t^{\tp}\, \hat{B}(i,j,\cdot),
\]
where \(z_t \in \mathbb{R}^q\) is the feature vector associated with the input \(t \in \mathcal{X}\).
More specifically, \(z_t\) can be constructed as
\(
z_t
\;=\;
\bigl[\, K(x_\ell, t) \,\bigr]_{x_\ell \in \bm{x}}^{\tp} \, Z^{+},
\)
where 
 \(Z^{+}\) denotes the Moore--Penrose pseudoinverse of \(Z\).

We solve \eqref{eqn:repar} using an accelerated alternating direction method of multipliers (ADMM) algorithm \citep{kadkhodaie2015accelerated}. The details of the algorithm can be found in Appendix \ref{sec:algo}. 

\begin{remark}
    The optimization problem for distributional matrix completion in \eqref{eqn: reg} shares some similarity with tensor completion methods \cite{liu2013tensor,mu2014square,gandy2011tensor}. However, there are several fundamental differences. First, most existing tensor completion approaches consider tensors whose modes are all finite-dimensional, whereas our distributional tensor contains an infinite-dimensional mode corresponding to probability distributions. Second, conventional tensor completion methods typically treat all modes symmetrically and assume that individual tensor entries are observed or missed independently. In contrast, for the distributional matrix $\mu$—which can be viewed as a tensor-valued element in the product space $\mathbb{R}^{m_1} \times \mathbb{R}^{m_2} \times \mathcal{H}$—missingness occurs at the level of entire fibers, namely $\mu(i,j,\cdot)$ for $i = 1,\ldots,m_1$ and $j = 1,\ldots,m_2$. These structural distinctions lead to substantial differences in the associated theoretical analysis.
On the other hand, there exist tensor-related works that allow one mode of the tensor to consist of functions \citep[e.g.,][]{han2024guaranteed,larsen2024tensor}. However, these works primarily focus on decomposition problems rather than tensor or matrix completion.
\end{remark}

\section{Theoretical Gaurantees}
\label{sec:theory}

In this section, we establish theoretical guarantees for the estimator $\hat{\mu}$.
Our analysis quantifies estimation error through the aggregate squared maximum mean discrepancy (MMD) across all matrix entries, both observed and unobserved:
\[
\sum_{i,j} \bigl\|\hat{\mu}(i,j,\cdot) - \mu^*(i,j,\cdot)\bigr\|_{\mathcal{H}}^2 .
\]
We begin by introducing notation that will be used throughout the section.
Define
$m = \min\{m_1,m_2\}$, 
$M = \max\{m_1,m_2\}$
and
$d = m_1+m_2$.
We write $\lesssim$ ($\gtrsim$) to denote inequalities holding up to an absolute positive multiplicative constant, and write $a \asymp b$ when both $a \lesssim b$ and $b \gtrsim a$ hold.
Finally, let
$N = \min_{i,j: T_{ij} =1} N_{ij}$
denote the minimum sample size among all \textit{observed} entries.
Next, we impose following technical assumptions.
\begin{assumption}
    \label{ass:kernel}
    There exists a constant $C_K>0$ such that $\sup_{x \in \calX} |K(x,x)| \leq C_K$.
\end{assumption}
\begin{assumption}
    \label{ass:data}
   The observation indicators $\{T_{ij}\}_{i=1,\dots, m_1, j = 1,\dots, m_2}$ are independent Bernoulli random variables with $\pi_{ij} = \Pr(T_{ij} = 1) > 0$.   $\{X^{(i,j)}_l\}_{l=1, \dots, N_{ij}}$ are independent samples generated from the distribution $\rho(i,j)$ and they are independent of $\{T_{ij}\}$. Also,   $\{X^{(i,j)}_l\}_{l=1, \dots, N_{ij}}$ are independent across different $(i,j)$.
\end{assumption}
 Define
$\pi_L= \min_{i,j}\pi_{ij}$ and $\pi_U = \max_{i,j}\pi_{ij}$. 
\begin{assumption}
    \label{ass:missing}
    There exists a constant $\xi>0$ such that $\pi_U/\pi_L \leq \xi$. In addition, $\pi_L^{-1} = \smallO(m/\log d)$.
\end{assumption}
\begin{assumption}
    \label{ass:sup}
    $\sup_{i,j} \|\mu^*(i,j,\cdot)\|_\calH \leq \alpha$.
\end{assumption}
Assumption~\ref{ass:kernel} imposes a mild regularity condition on the kernel function and is introduced to control the convergence of the empirical KME. 
Assumption~\ref{ass:data} specifies the data-generating process. The corresponding missing structure allows non-uniform missingness and is standard in the matrix completion literature \citep[e.g.][]{klopp,negahban2012restricted,wang2021matrix,zhao2025noisy}.
Assumption~\ref{ass:missing} is a commonly adopted condition for nuclear-norm-based matrix completion estimators \citep[e.g.][]{klopp,negahban2012restricted,koltchinskii2011nuclear,wang2021matrix,li2024pairwise}.
The first part of this assumption rules out scenarios in which the observation probabilities differ in their asymptotic orders.
The second part is a mild technical requirement to help establish the consistency of the estimator.
Finally, Assumption~\ref{ass:sup} is a realization assumption, so that the ground truth satisfies the constraint in our optimization~\eqref{eqn: reg}.

First, we derive Theorem \ref{thm:approximately} for the case when $\mu^*$ has approximately low multilinear rank, i.e., $\|\calU_j\mu^*\|_* \lesssim \sqrt{m_1m_2}$.

\begin{theorem}
    \label{thm:approximately}
     Under Assumptions \ref{ass:kernel}-\ref{ass:sup}, if we take
\begin{align*}
   & (\lambda \beta_1  + \lambda \beta_2)^{-1}=  \\
   & \qquad  \bigO \left( \left[\frac{\sqrt{M\pi_U \log d}}{m_1 m_2}\max\left\{ 
 \Delta_N \sqrt{\log d}
 ,  \alpha \right\} \right]^{-1}\right), 
\end{align*}
with $\Delta_N = C_K/N + \sqrt{\log d/N}$,
then, with probability at least $1- \kappa/d$ for some constant $\kappa>0$, we have
\begin{multline*}
   \frac{1}{m_1m_2} \sum_{i,j} \|\hat{\mu}(i,j,\cdot)- {\mu}^*(i,j,\cdot)\|_\calH^2   \lesssim\\
  ({\lambda}/{\pi_L})\left(   \beta_1  \|\calU_1(\mu^*)\|_*+ \beta_2 \|\calU_2(\mu^*)\|_* + \beta_3 \|\calU_3(\mu^*)\|_* \right)
  + (2\alpha)^2 \sqrt{\frac{\pi_U \log d}{\pi_L^2m_1m_2}}.\\ 
\end{multline*}
More specifically, if we take $\lambda\beta_1, \lambda\beta_2 \asymp \frac{\sqrt{M\pi_U \log d}}{m_1 m_2}\max\left\{ 
 \Delta_N \sqrt{\log d}
 ,  \alpha \right\}$, then, with probability at least $1-\kappa/d$,
 we have
  \begin{multline*}
    \frac{1}{m_1m_2} \sum_{i,j} \|\hat{\mu}(i,j,\cdot)- {\mu}^*(i,j,\cdot)\|_\calH^2 \lesssim   \\
    \max\{\alpha , \Delta_N \sqrt{\log d}\} \sqrt{\frac{\pi_U \log d}{\pi_L^2 m}} \left\{ \frac{ \|\calU_1(\mu^*)\|_*}{\sqrt{m_1m_2}} +  \right.\left.\frac{ \|\calU_2(\mu^*)\|_*}{\sqrt{m_1m_2}} \right\}
     +  (2\alpha)^2 \sqrt{\frac{\pi_U \log d}{\pi_L^2m_1m_2}} + \frac{\lambda \beta_3}{\pi_L} \|\calU_3(\mu^*)\|_*.  \label{eqn:bound1}
  \end{multline*}
\end{theorem}

 By choosing \(\lambda \beta_3\) to be sufficiently small, the last term in \eqref{eqn:bound1} becomes negligible.
On the other hand, assigning a larger weight to \(\beta_3\) promotes a reduction in the number of basis functions required to represent the KMEs, which is often beneficial in terms of downstream analysis and interpretability.
When the first term in \eqref{eqn:bound1} dominates, the resulting error bound matches the rate established for classical nuclear-norm–regularized matrix completion of approximately low-rank matrices \citep{negahban2012restricted}, despite the more challenging setting considered here, where each matrix entry corresponds to a probability distribution rather than a scalar.
In our technical analysis, we develop novel lemmas to tackle these unfolding operators. See Lemma \ref{lem:noise}  and \ref{lem:frobenius}  for details.

Recall that, 
to our knowledge, \citet{feitelberg2024distributional} is the only alternative method for distributional matrix completion.
Compared to the theoretical results in \citet{feitelberg2024distributional}, our analysis differs in two important aspects.
First, the bounds in \citet{feitelberg2024distributional} are {pointwise}, in the sense that they characterize the estimation error at a fixed location \((i,j)\).
Second, their consistency guarantees require  \(N\) to diverge.
In contrast, our results provide global guarantees over the entire matrix and do not require \(N\) to grow. More specifically, the term \(\Delta_N\) captures the convergence rate of the empirical kernel mean embedding with respect to the per-entry sample size \(N\).
Our bound indicates that consistency of the final estimator as long as $m$ diverges. 
This phenomenon arises from the intrinsic low-rank structure of \(\mu^\ast\), which enables effective information sharing across observed entries.
As a result, the proposed estimator not only imputes distributions for unobserved entries but also yields more accurate estimates for observed entries with limited sample sizes.

Next, we consider the setting when $\mu^*$ has exactly low multilinear rank. 

\begin{theorem}
    \label{thm:low}
   Under Assumptions \ref{ass:kernel}-\ref{ass:sup}, suppose that $\mu^*$ has multilinear rank $(r_1, r_2, r_3)$, where $r_1, r_2, r_3$ are of constant orders.  
    By taking 
\begin{align*}
   & (\lambda \beta_1  + \lambda \beta_2)^{-1}=  
 \bigO \left( \left[\frac{\sqrt{M\pi_U \log d}}{m_1 m_2}\max\left\{ 
 \Delta_N \sqrt{\log d}
 ,  \alpha \right\} \right]^{-1}\right), 
\end{align*}
we have with probability at least $1- \kappa/d$ for some constant $\kappa>0$, 
\begin{multline}
    \frac{1}{m_1m_2} \sum_{i,j} \|\hat \mu(i,j,\cdot) - \mu(i,j,\cdot)\|^2_F \lesssim\\  \min\left\{\frac{\lambda^2 \beta_1^2m_1m_2r_1}{\pi_L^2} + \frac{\lambda^2\beta_2^2m_1m_2r_2}{\pi_L^2} , \frac{\lambda^2\beta_3^2 m_1m_2r_3}{\pi_L^2} \right\} 
    + (2\alpha)^2 \sqrt{\frac{\pi_U \log d}{\pi_L^2m_1m_2}}.  \nonumber   
\end{multline}
Furthermore, if we take $\lambda\beta_1, \lambda\beta_2, \lambda\beta_3 \asymp \frac{\sqrt{M\pi_U \log d}}{m_1 m_2}\max\left\{ 
 \Delta_N \sqrt{\log d}
 ,  \alpha \right\}$, then with probablity at least $1-\kappa/d$,
 we have
  \begin{multline*}
     \frac{1}{m_1m_2} \sum_{i,j} \|\hat \mu(i,j,\cdot) - \mu(i,j,\cdot)\|^2_F  \lesssim\\
        \frac{\pi_U}{\pi_L^2} \frac{[\max\{\alpha,\Delta_N\}]^2 \min\{r_1, r_2, r_3\}}{m}\log^2 d
        + (2\alpha)^2 \sqrt{\frac{\pi_U \log d}{\pi_L^2m_1m_2}}.
  \end{multline*}
\end{theorem}
Theorem~\ref{thm:low} improves upon the convergence rate established in Theorem~\ref{thm:approximately} by fully exploiting the exact low-rank structure of $\mu^*$.
Moreover, the resulting rate matches that obtained for classical noisy low-rank matrix completion with nuclear-norm regularization \citep{klopp,negahban2012restricted}, which is known to be minimax optimal up to logarithmic factors in the scalar-valued setting.

\section{Simulation Studies}
\label{sec:simu}


In this section, we present numerical experiments to illustrate the empirical performance of the proposed method.
We generate synthetic data according to the following procedure.
We first construct a matrix of distributions of dimension $m_1\times m_2$.
Fix $r_1=r_2=3$.
For each simulation run, we generate parameters
$a_l, v_l \overset{\mathrm{i.i.d.}}{\sim} \mathrm{Unif}[0,5]$ for $\l=1,\dots,r_1r_2$.
Let $g=[0.5,0.5,0.5]^\tp$, and draw mixing weights
$w_k \overset{\mathrm{i.i.d.}}{\sim} \mathrm{Dirichlet}(g)$ for $k=1,\dots,m_1$
and
$w'_{k'} \overset{\mathrm{i.i.d.}}{\sim} \mathrm{Dirichlet}(g)$ for $k'=1,\dots,m_2$.
Each matrix entry $(i,j)$ is then assigned a Gaussian mixture distribution
\[
 \sum_{l_1=1}^{r_1} \sum_{l_2=1}^{r_2} w_{i,l_1} w'_{j,l_2} \mathcal{N}(a_{l_2 + r_1(l_1-1)}, v_{l_2 + r_1(l_1-1)}).
\]
Next, we generate an observation indicator matrix $T$ by uniformly sampling entries with observation probability $p$.
For each observed entry $(i,j)$ with $T(i,j)=1$, we draw $N$ i.i.d.~samples from the corresponding distribution $\rho(i,j)$.

We compare our proposed estimator (\prop)  with an existing estimator \citet{feitelberg2024distributional} (\NN) using the following relative metrics defined on the whole matrix and unobserved entries:
\begin{align*}
    \mathrm{Re}(\mathrm{MMD}) = \frac{\sum_{i,j} \| \mu(i,j,\cdot) - \hat{\mu} (i,j,\cdot) \|^2_\calH}{\sum_{i,j} \| \mu(i,j,\cdot) \|^2_\calH},\\
    \mathrm{Re}_{t}(\mathrm{MMD}) = \frac{\sum_{T(i,j)=0} \| \mu(i,j,\cdot) - \hat{\mu} (i,j,\cdot) \|^2_\calH}{\sum_{T(i,j)=0} \| \mu(i,j,\cdot) \|^2_\calH},
\end{align*}
where $\hat{\mu}$ is a general notation for the estimated KME.  Beyond distributional accuracy, various summary statistics can be derived from the estimated \(\hat{\mu}\) (or \(\hat{\rho}\))
for evaluation purpose.
In particular, we compute the estimated mean matrix \(\hat{M}\) and variance matrix \(\hat{V}\), and compare them with the corresponding population mean and variance matrices \(M\) and \(V\). We also compute mean and variance estimators using classical matrix completion with nuclear-norm regularization (\soft), treating the sample means and sample variances as observed scalar entries.
The relative Frobenius norm errors are defined as
\small
\begin{align*}
     \mathrm{Re}(M) = \frac{ \| M  - \hat{M} \|^2_F}{ \| M\|^2_F};
    \mathrm{Re}_{t}(M) = \frac{ \| (1-T)(M  - \hat{M}) \|^2_F}{ \| (1-T)M\|^2_F};\\
    \mathrm{Re}(V) = \frac{ \|V  - \hat{V} \|^2_F}{ \| V\|^2_F};
    \mathrm{Re}_{t}(V) = \frac{ \| (1-T)(V  - \hat{V}) \|^2_F}{ \| (1-T)V\|^2_F};
\end{align*}
\normalsize
where $\|A\|_F$ refers to the Frobenius norm of matrix $A$.  In our implementation, we employ the Gaussian kernel to construct KMEs and select the kernel bandwidth using the widely adopted {median heuristic} \citep{garreau2017large}.
For simplicity in tuning parameter selection, we set \(\beta_j =1/3, j=1,2,3,\) in \eqref{eqn:reg_nuclear} and choose the regularization parameter \(\lambda\) via cross-validation.
Cross-validation is also used to select tuning parameters for all competing methods considered in our experiments.

Table~\ref{tab:1dim} summarizes the results for the setting $m_1 = m_2 = 100$, $p = 0.6$, and $N = 5$, based on $100$ simulated datasets. As shown, the proposed method $\prop$ consistently outperforms $\NN$ across all evaluation metrics, with particularly notable improvements on the unobserved entries. For both mean and variance matrix estimation, $\prop$ also achieves better performance than $\soft$, which may be attributed to its ability to incorporate richer distributional information and to exploit additional structural properties of the data. While $\NN$ performs reasonably well in estimating the mean, it fails to deliver competitive variance estimates.

In addition to the one-dimensional setting, we also consider a two-dimensional setting.
Specifically, we construct mean vectors \(a_\ell \overset{\text{i.i.d.}}{\sim} \mathrm{Unif}([0,1]^2)\) and covariance matrices \(V_\ell = A A_\ell^\top\), where \(A_\ell \in \mathbb{R}^{2 \times 2}\) has entries independently generated from a standard Gaussian distribution, $l = 1,\dots, r_1r_2$. 
We adopt analogous relative error metrics for evaluation.
In particular, we report relative errors on the entire matrix for the first and second components of the mean matrix, denoted by \(\mathrm{Re}(M1)\) and \(\mathrm{Re}(M2)\), as well as for the covariance matrix, denoted by \(\mathrm{Re}(\mathrm{Cov})\).
Relative errors restricted to the unobserved entries are defined in the same manner.
For the method \(\NN\), however, computing the two-dimensional Wasserstein distance is considerably more challenging and does not readily extend to this setting.
As a result, we exclude \(\NN\) from the comparison in the two-dimensional experiments.
As Table \ref{tab:2dim} shows, \prop{} performs better than \soft by utilizing more structural information on the estimation of summary statistics.


\begin{table}[t]
  \caption{Simulation results with one-dimensional samples for $m_1=m_2=100$, $p=0.6$ and $N= 5$. Values in the parentheses are standard errors. }
  \label{tab:1dim}
  \centering
    \begin{tabular}{lccc}
      \toprule
      Method & \prop & \soft & \NN \\
      \midrule
      $\mathrm{Re}(\mathrm{MMD})$ & 0.064 (0.0006) & N/A & 0.090 (0.0052)\\
      $\mathrm{Re}(M)$            & 0.044 (0.0048) & 0.120 (0.0160) & 0.056 (0.0053) \\
      $\mathrm{Re}(V)$            & 0.076 (0.0029) & 0.305 (0.0064) & 0.177 (0.0052) \\
      \bottomrule
        $\mathrm{Re}_t(\mathrm{MMD})$ & 0.039 (0.0005) & N/A & 0.090 (0.0052)  \\
      $\mathrm{Re}_t(M)$            & 0.029 (0.0025) & 0.050 (0.0076) & 0.056 (0.0053) \\
      $\mathrm{Re}_t(V)$            &  0.044 (0.0043) & 0.186 (0.0037) & 0.177 (0.0052) \\
      \bottomrule
    \end{tabular}
\end{table}
\begin{table}[t]
  \caption{Simulation results with two-dimensional samples for $m_1=m_2=50$, $p=0.3$ and $N= 80$. Values in the parentheses are standard errors. }
  \label{tab:2dim}
  \centering
    \begin{tabular}{lcc}
      \toprule
      Method & \prop & \soft \\
      \midrule
      $\mathrm{Re}(\mathrm{MMD})$ & 0.0161 (0.0001) & N/A\\
      $\mathrm{Re}(M1)$            & 0.0201 (0.0011)&  0.0349 (0.0014)  \\
      $\mathrm{Re}(M2)$& 0.0183 (0.0008) & 0.0346 (0.0016) \\
      $\mathrm{Re}(Cov)$            & 0.2129 (0.0141) & 0.2337 (0.0091)\\
      \bottomrule
        $\mathrm{Re}_t(\mathrm{MMD})$ & 0.0156 (0.0001) &  N/A\\
      $\mathrm{Re}_t(M1)$            & 0.0179 (0.0010) & 0.0327 (0.0014)  \\
      $\mathrm{Re}_t(M2)$& 0.0158 (0.0008) & 0.0322 (0.0015) \\
      $\mathrm{Re}_t(Cov)$            & 0.1782 (0.0131) &  0.2350 (0.0092) \\
      \bottomrule
    \end{tabular}
  \vspace{-0.4cm}
\end{table}

\section{Real Data Application}
\label{sec:real}

In this section, we analyze the NYC taxi trip dataset
(\url{https://www.nyc.gov/site/tlc/about/tlc-trip-record-data.page})
to illustrate the performance of the proposed method.
We adopt the cleaned and organized data directly from the public repository
\url{https://github.com/wooner49/sofia}.
The processed dataset records source locations, destination locations, and the number of taxi trips per day between each source-destination pair over multiple days.


We treat source and destination locations as matrix indices, yielding a $265 \times 265$ matrix, where each entry represents the distribution of daily trip counts between a given source--destination pair. Due to data collection and preprocessing issues---such as missing or invalid location identifiers, timestamps, or corrupted logs---many entries are unobserved. The overall missing rate exceeds $39\%$, and the number of samples per observed entry varies widely, ranging from $1$ to $3672$. For evaluation, we randomly select $90\%$ of the observed entries for training and reserve the remaining $10\%$ as the test set $\mathcal{T}$.

We compared the two distributional matrix completion methods \prop{} (our proposal) and \NN{} based on
the empirical relative error
\begin{align*}
&\hat{\mathrm{Re}}_t(\mathrm{MMD})\\
&=
      \frac{\sum_{(i,j)\in \calT}\|\frac{1}{N_{ij}}\sum_{l=1}^{N_{ij}}K(X_{l},\cdot) - \hat{\mu}(i,j,\cdot)\|_\calH^2}{\sum_{(i,j)\in \calT}\|\sum_{l=1}^{N_{ij}}K(X_{l},\cdot)/N_{ij}\|_\calH^2},
\end{align*}
where $\hat{\mu}$ is a general notation for the estimated KME.  As Table \ref{tab:real} shows,  \prop{} achieves much smaller $\hat{\mathrm{Re}}_t(\mathrm{MMD})$ than \NN. 
\begin{table}[h]
\vspace{-0.2cm}
  \caption{Empirical relative error for the test set.}
  \label{tab:real}
  \centering
    \begin{tabular}{lcc}
      \toprule
      Method & \prop & \NN \\
      \midrule
      $\hat{\mathrm{Re}}_t(\mathrm{MMD})$ & 0.0015 & 0.0212 \\
      \bottomrule
    \end{tabular}
  \vspace{-0.4cm}
\end{table}

We take a closer look at two representative test cases in which the number of observed samples exceeds 50.
The first case exhibits a relatively heavy-tailed distribution, as shown in Figure~\ref{fig:case1}(a), with values ranging from 1 to above 25, whereas the second case is highly concentrated around 1, as illustrated in Figure~\ref{fig:case2}(a).

For test case~1, the KME estimated by \(\prop\) closely matches the shape of the empirical KME.
In particular, both display relatively large values in the tail region when \(x\) is large, whereas the KME estimated by \(\NN\) is predominantly concentrated at smaller values.
To facilitate a more direct comparison, we further transform the KME estimated by \(\prop\) into a probability density function via numerical inversion (see Appendix~\ref{sec:num_inversion} for details).
As shown in Figure~\ref{fig:case1}(d), the resulting density function from \(\prop\) exhibits a pattern consistent with the histogram of the observed samples.
In contrast, the histogram derived from the empirical samples used by \(\NN\) (Figure \ref{fig:case1}(c)) is largely concentrated at value 1 and fails to capture the distributional mass beyond 11.

For test case~2, the KME estimated by \(\prop\) exhibits the same overall trend as the empirical KME, attaining its maximum at value 1 and then decreasing.
In contrast, the KME estimated by \(\NN\) reaches its maximum at a value greater than 1.5.
This discrepancy is also evident from the histogram of the empirical samples used by \(\NN\) (Figure~\ref{fig:case2}(c)), which has its mode at value 2 and extends beyond value 3.
By comparison, the density estimated from \(\prop\) (Figure~\ref{fig:case2}(d)) assigns a dominantly high density to value 1, more accurately reflecting the underlying empirical distribution.

\begin{figure}[ht]
  \centering
  \begin{subfigure}{0.48\linewidth}
    \includegraphics[width=\linewidth]{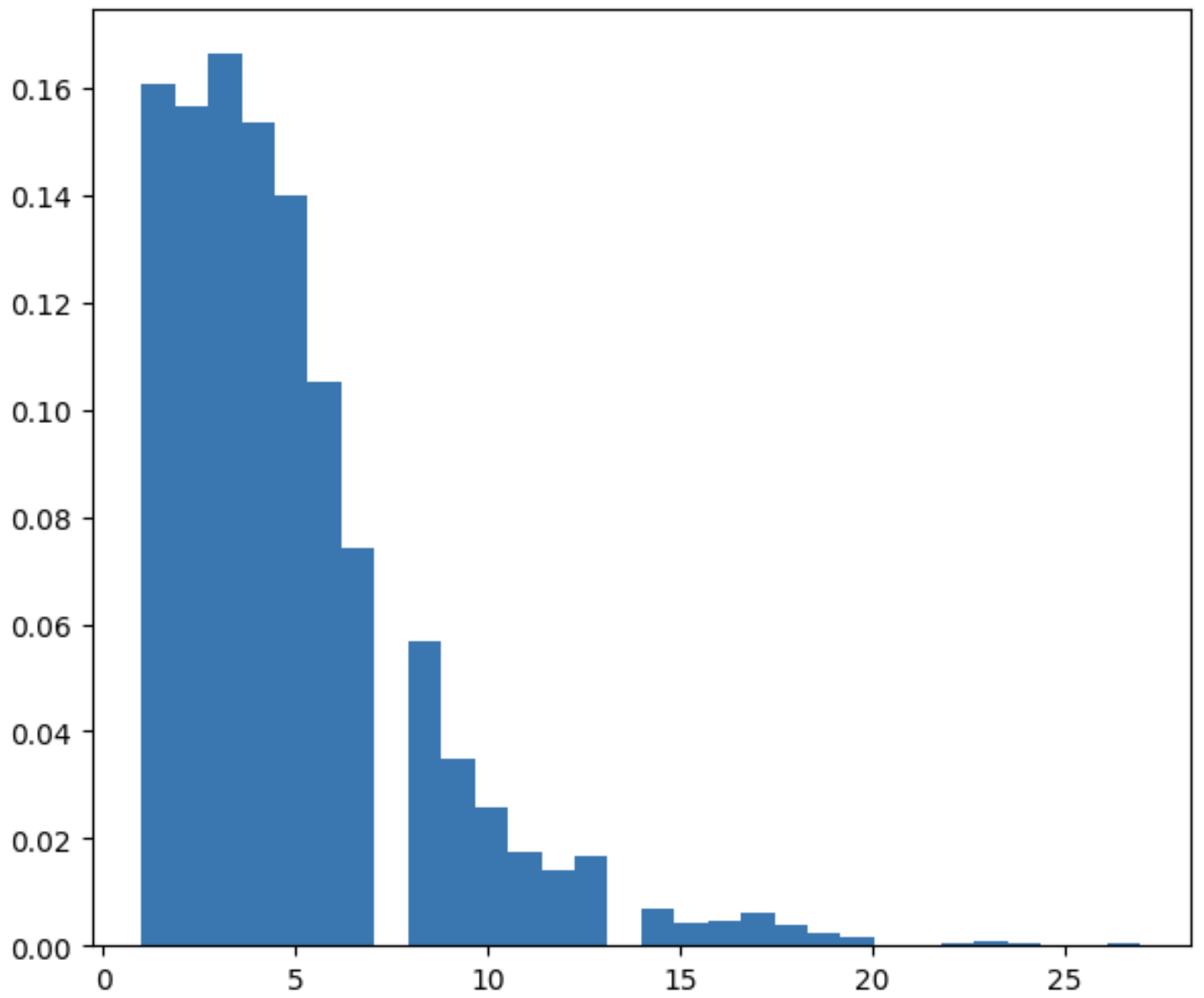}
    \caption{Histogram for observed samples.}
  \end{subfigure}
  \hfill
  \begin{subfigure}{0.48\linewidth}
    \includegraphics[width=\linewidth]{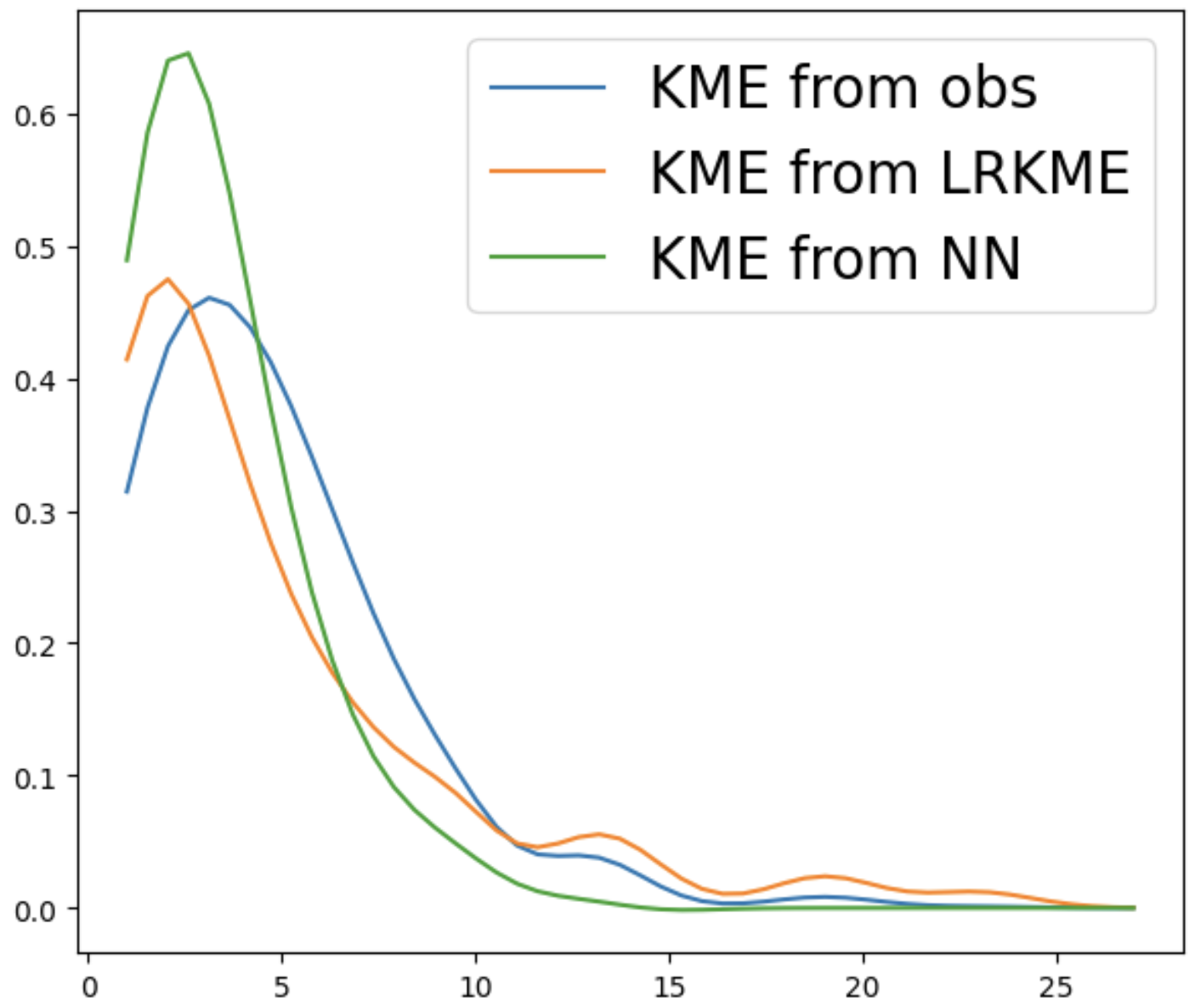}
    \caption{Empirical KME, estimated KMEs from \NN and \prop{}.}
  \end{subfigure}

  \medskip

  \begin{subfigure}{0.48\linewidth}
    \includegraphics[width=\linewidth]{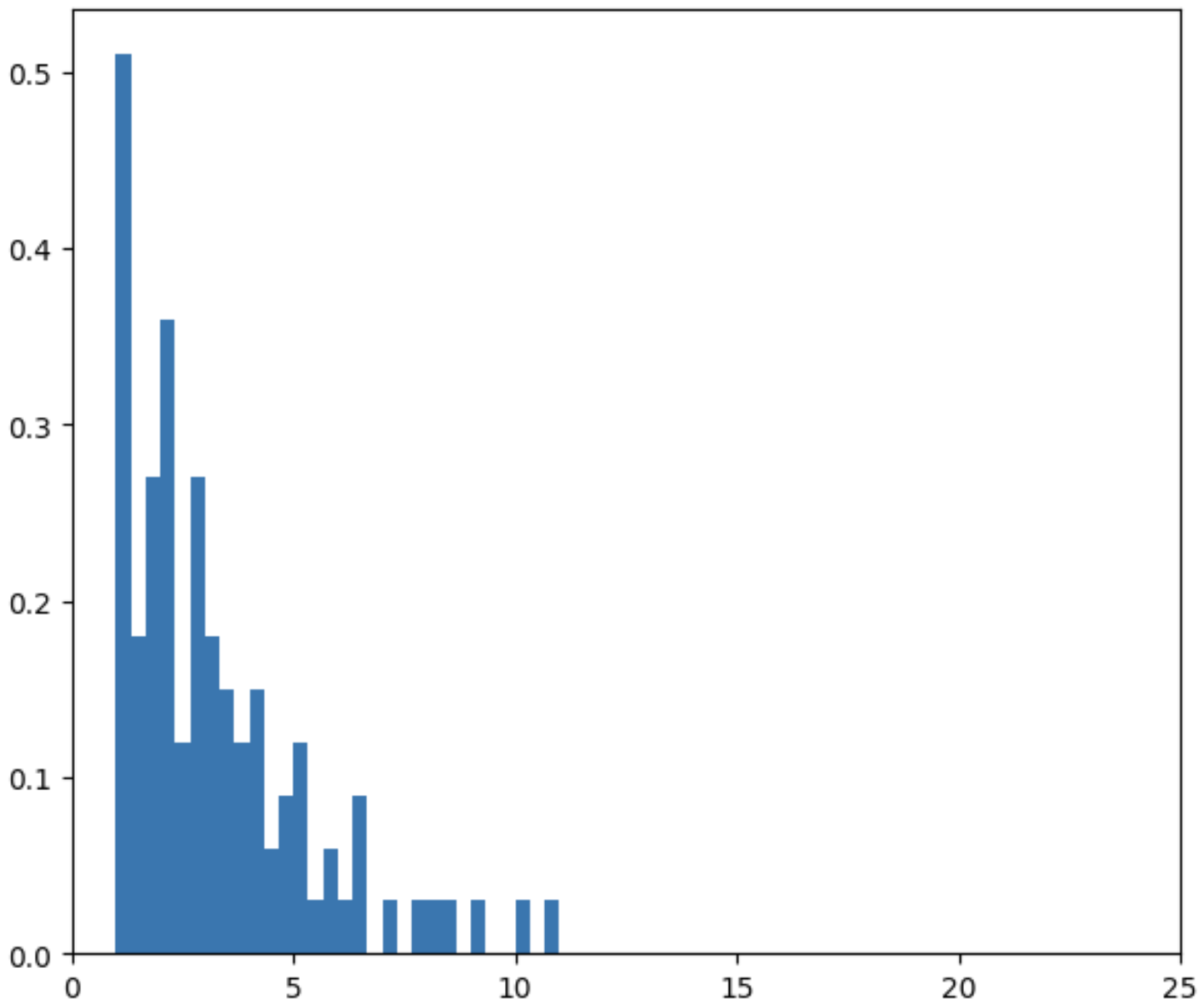}
    \caption{Histogram for empirical samples estimated from \NN. }
  \end{subfigure}
  \hfill
  \begin{subfigure}{0.48\linewidth}
    \includegraphics[width=\linewidth]{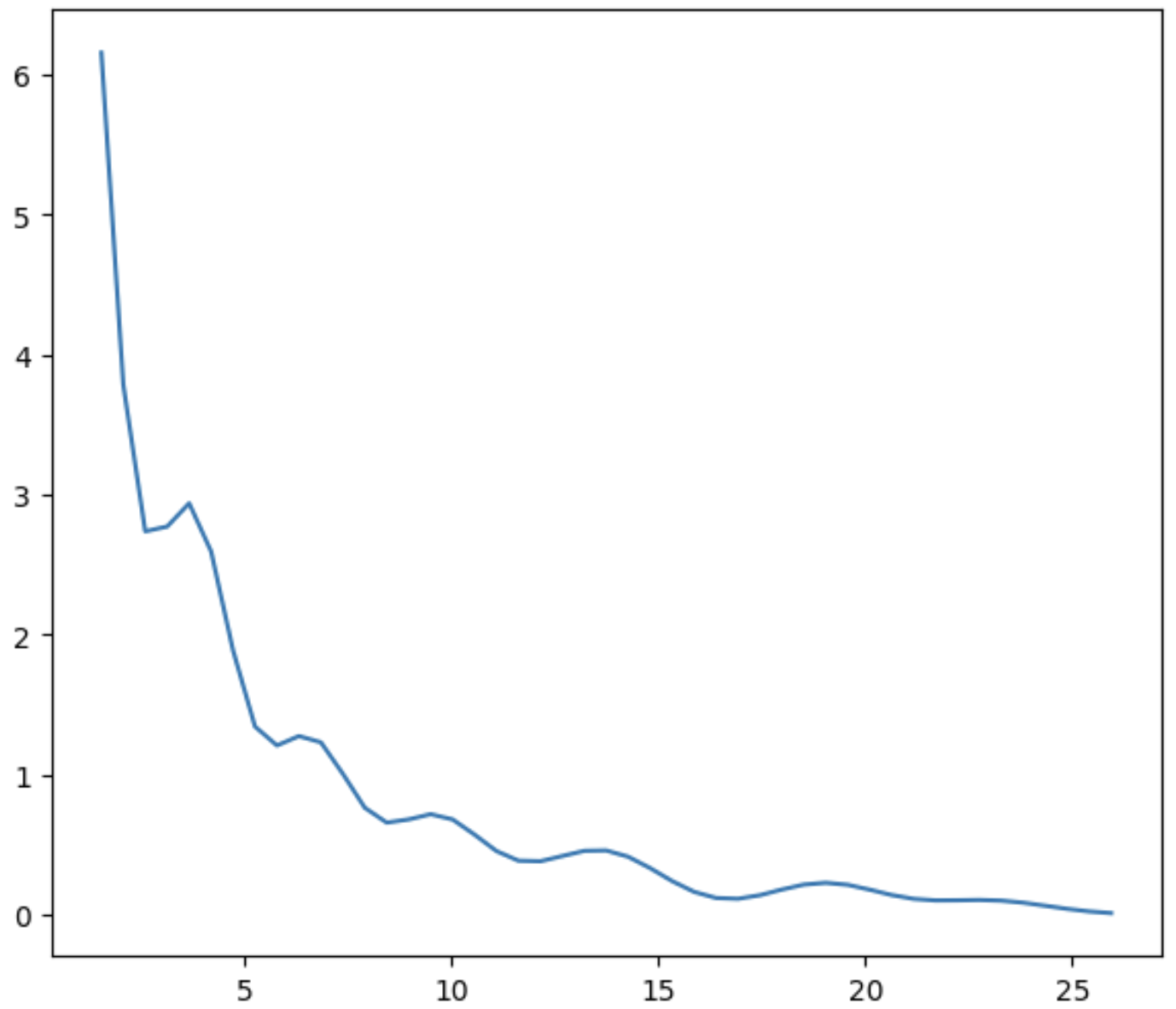}
    \caption{Estimated density from estimated KME from \prop.}
  \end{subfigure}
  \caption{Results for test case 1.}
  \label{fig:case1}
\end{figure}

\begin{figure}[H]
  \centering
  \begin{subfigure}{0.48\linewidth}
    \includegraphics[width=\linewidth]{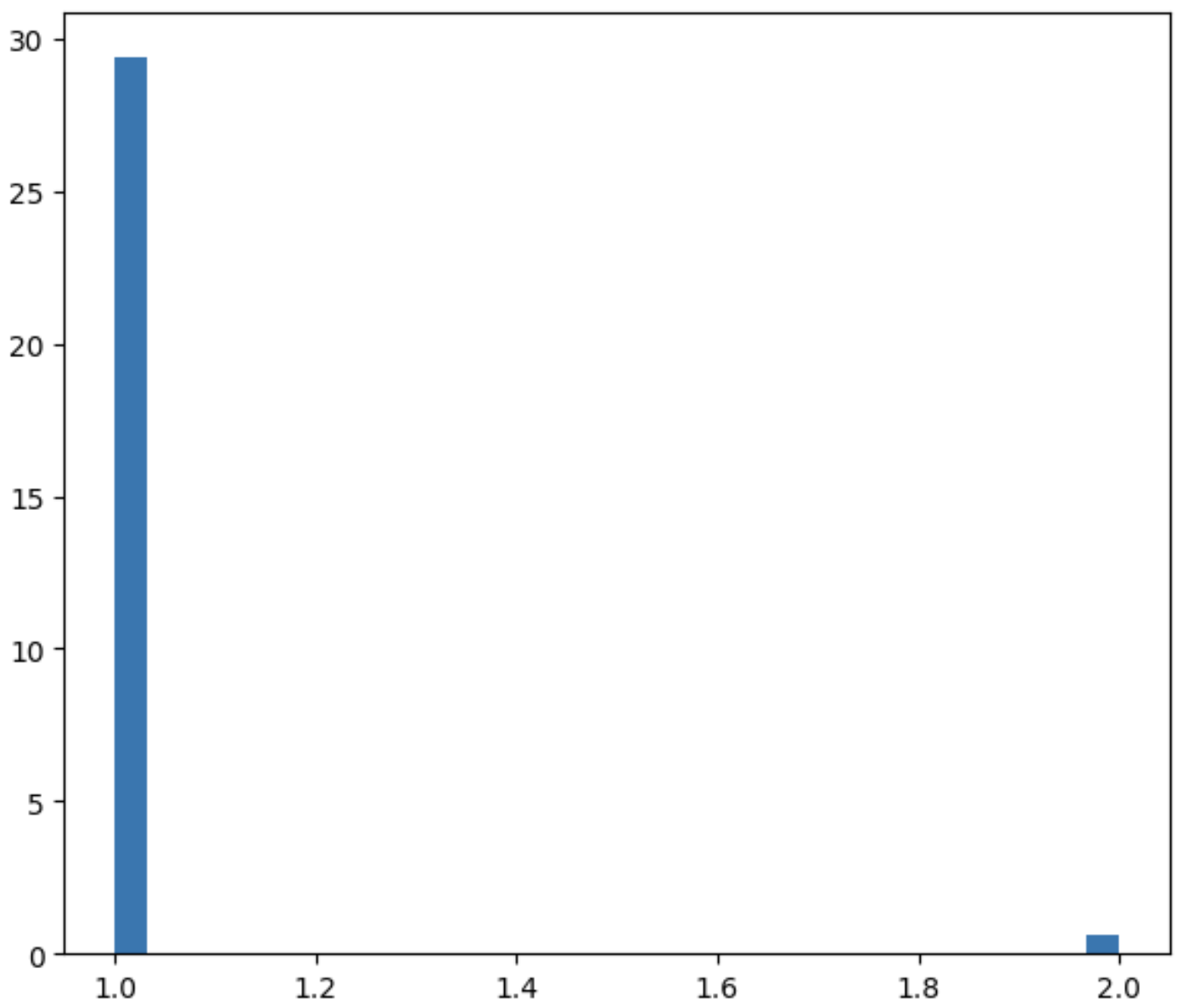}
    \caption{Histogram for observed samples.}
  \end{subfigure}
  \hfill
  \begin{subfigure}{0.48\linewidth}
    \includegraphics[width=\linewidth]{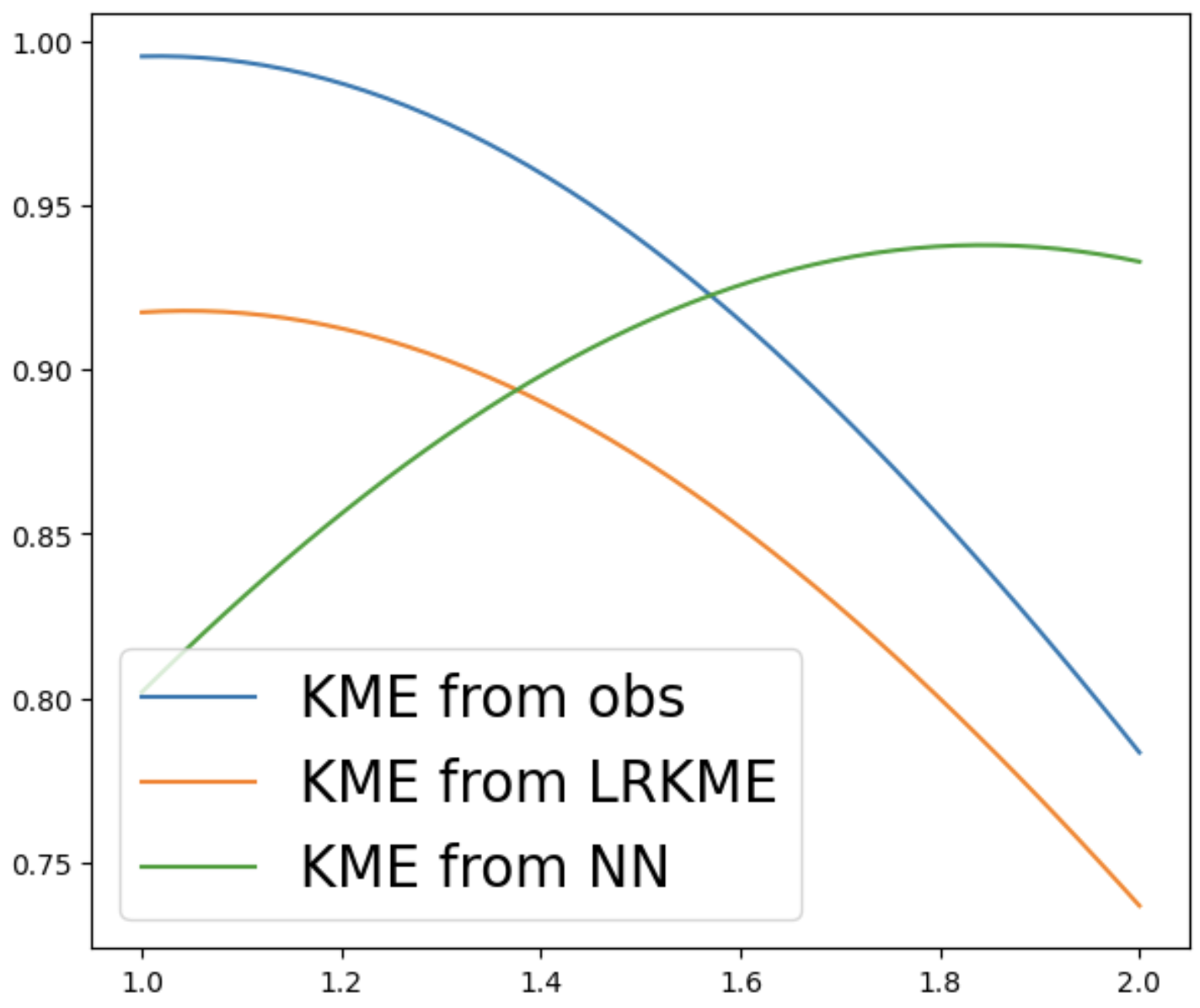}
    \caption{Empirical KME, estimated KMEs from \NN and \prop{}.}
  \end{subfigure}
  \medskip
  \begin{subfigure}{0.48\linewidth}
    \includegraphics[width=\linewidth]{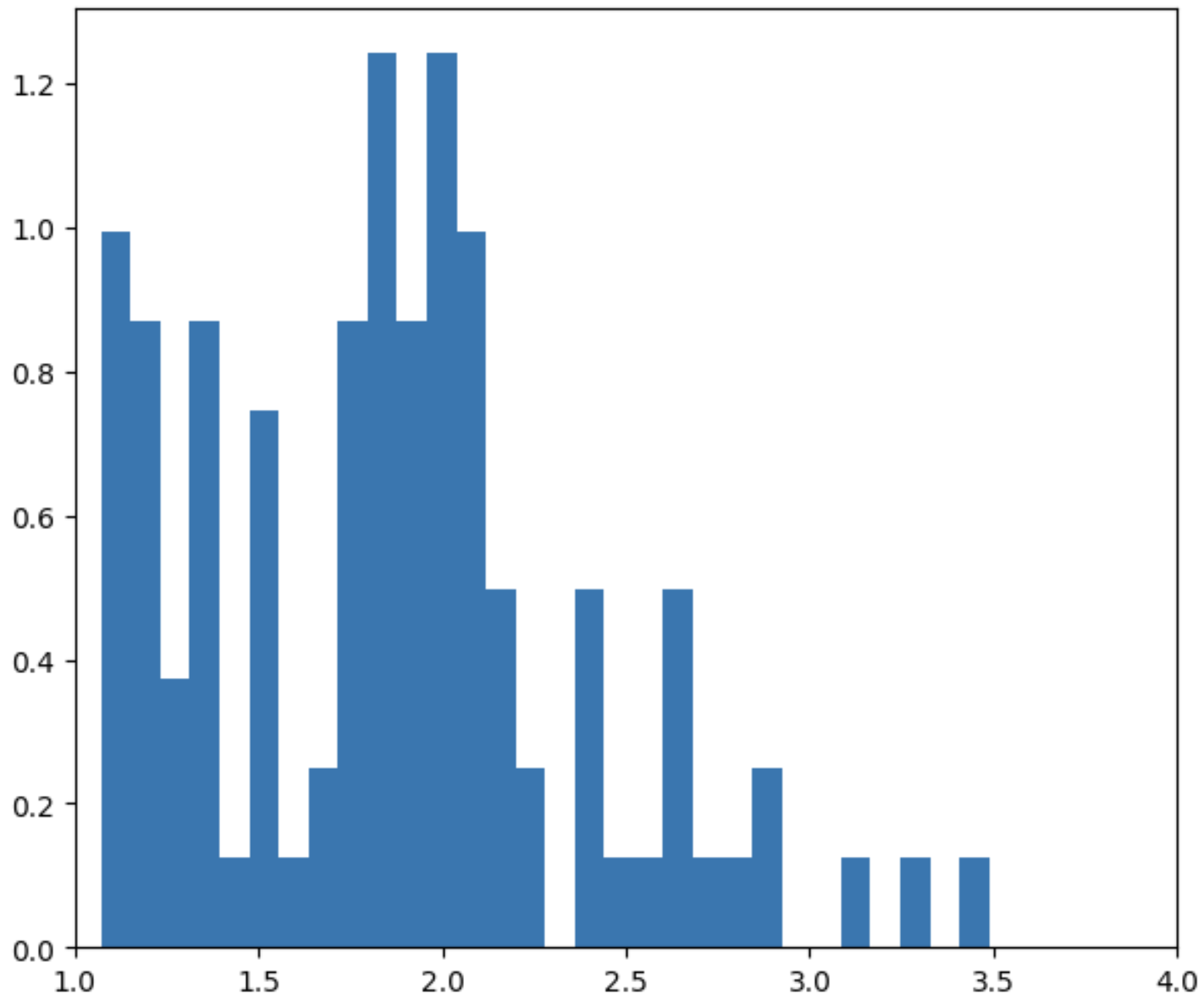}
    \caption{Histogram for empirical samples estimated from \NN. }
  \end{subfigure}
  \hfill
  \begin{subfigure}{0.48\linewidth}
    \includegraphics[width=\linewidth]{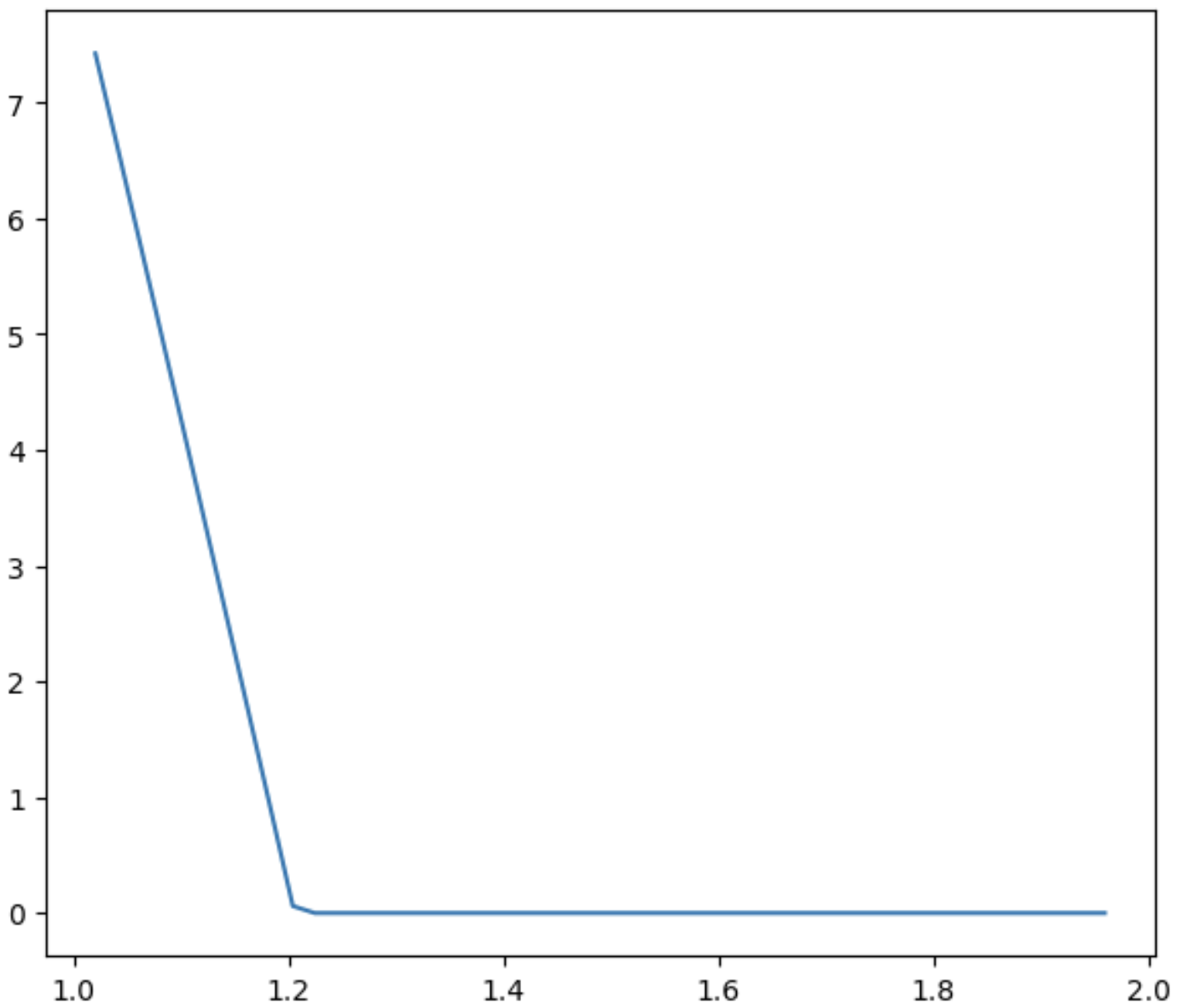}
    \caption{Estimated density from estimated KME from \prop.}
  \end{subfigure}
  \caption{Results for test case 2.}
  \label{fig:case2}
\end{figure}

\bibliographystyle{plainnat}
\bibliography{refs}

\begin{thebibliography}{43}
\providecommand{\natexlab}[1]{#1}
\providecommand{\url}[1]{\texttt{#1}}
\expandafter\ifx\csname urlstyle\endcsname\relax
  \providecommand{\doi}[1]{doi: #1}\else
  \providecommand{\doi}{doi: \begingroup \urlstyle{rm}\Url}\fi

\bibitem[Boyd et~al.(2010)Boyd, Parikh, Chu, Peleato, and
  Eckstein]{Boyd-Parikh-Chu10}
Stephen Boyd, Neal Parikh, Eric Chu, Borja Peleato, and Jonathan Eckstein.
\newblock Distributed optimization and statistical learning via the alternating
  direction method of multipliers.
\newblock \emph{Foundations and Trends{\textregistered} in Machine Learning},
  3\penalty0 (1):\penalty0 1--122, 2010.

\bibitem[Candes and Plan(2010)]{candes2010matrix}
Emmanuel~J Candes and Yaniv Plan.
\newblock Matrix completion with noise.
\newblock \emph{Proceedings of the IEEE}, 98\penalty0 (6):\penalty0 925--936,
  2010.

\bibitem[Cao et~al.(2014)Cao, Cai, and Tan]{cao2014image}
Feilong Cao, Miaomiao Cai, and Yuanpeng Tan.
\newblock Image interpolation via low-rank matrix completion and recovery.
\newblock \emph{IEEE Transactions on Circuits and Systems for Video
  Technology}, 25\penalty0 (8):\penalty0 1261--1270, 2014.

\bibitem[Changjun et~al.(2012)Changjun, Xiangyang, Yongbing, and
  Qionghai]{changjun2012single}
Fu~Changjun, Ji~Xiangyang, Zhang Yongbing, and Dai Qionghai.
\newblock A single frame super-resolution method based on matrix completion.
\newblock In \emph{2012 Data Compression Conference}, pages 297--306. IEEE,
  2012.

\bibitem[Chen et~al.(2023)Chen, Wang, Ng, and Wang]{chen2023high}
Junren Chen, Cheng-Long Wang, Michael~K Ng, and Di~Wang.
\newblock High dimensional statistical estimation under uniformly dithered
  one-bit quantization.
\newblock \emph{IEEE Transactions on Information Theory}, 2023.

\bibitem[Chen et~al.(2017)Chen, Wei, Li, Liang, Cai, and
  Zhang]{chen2017ensemble}
Xiaobo Chen, Zhongjie Wei, Zuoyong Li, Jun Liang, Yingfeng Cai, and Bob Zhang.
\newblock Ensemble correlation-based low-rank matrix completion with
  applications to traffic data imputation.
\newblock \emph{Knowledge-Based Systems}, 132:\penalty0 249--262, 2017.

\bibitem[Chen et~al.(2020)Chen, Yang, and Sun]{chen2020nonconvex}
Xinyu Chen, Jinming Yang, and Lijun Sun.
\newblock A nonconvex low-rank tensor completion model for spatiotemporal
  traffic data imputation.
\newblock \emph{Transportation Research Part C: Emerging Technologies},
  117:\penalty0 102673, 2020.

\bibitem[Chen et~al.(2021)Chen, Zhao, and Wang]{chen2021kernel}
Zhaoliang Chen, Wei Zhao, and Shiping Wang.
\newblock Kernel meets recommender systems: A multi-kernel interpolation for
  matrix completion.
\newblock \emph{Expert Systems with Applications}, 168:\penalty0 114436, 2021.

\bibitem[Drineas and Mahoney(2005)]{drineas2005nystrom}
Petros Drineas and Michael~W. Mahoney.
\newblock On the nystr{\"o}m method for approximating a gram matrix for
  improved kernel-based learning.
\newblock \emph{Journal of Machine Learning Research}, 6:\penalty0 2153--2175,
  2005.

\bibitem[Feitelberg et~al.(2024)Feitelberg, Choi, Agarwal, and
  Dwivedi]{feitelberg2024distributional}
Jacob Feitelberg, Kyuseong Choi, Anish Agarwal, and Raaz Dwivedi.
\newblock Distributional matrix completion via nearest neighbors in the
  wasserstein space.
\newblock \emph{arXiv preprint arXiv:2410.13112}, 2024.

\bibitem[Gandy et~al.(2011)Gandy, Recht, and Yamada]{gandy2011tensor}
Silvia Gandy, Benjamin Recht, and Isao Yamada.
\newblock Tensor completion and low-n-rank tensor recovery via convex
  optimization.
\newblock \emph{Inverse Problems}, 27\penalty0 (2):\penalty0 025010, 2011.

\bibitem[Garreau et~al.(2017)Garreau, Jitkrittum, and
  Kanagawa]{garreau2017large}
Damien Garreau, Wittawat Jitkrittum, and Motonobu Kanagawa.
\newblock Large sample analysis of the median heuristic.
\newblock \emph{arXiv preprint arXiv:1707.07269}, 2017.

\bibitem[Gretton et~al.(2012)Gretton, Borgwardt, Rasch, Sch{\"o}lkopf, and
  Smola]{gretton2012kernel}
Arthur Gretton, Karsten~M. Borgwardt, Malte~J. Rasch, Bernhard Sch{\"o}lkopf,
  and Alex~J. Smola.
\newblock A kernel two-sample test.
\newblock \emph{Journal of Machine Learning Research}, 13:\penalty0 723--773,
  2012.

\bibitem[Gurini et~al.(2018)Gurini, Gasparetti, Micarelli, and
  Sansonetti]{gurini2018temporal}
Davide~Feltoni Gurini, Fabio Gasparetti, Alessandro Micarelli, and Giuseppe
  Sansonetti.
\newblock Temporal people-to-people recommendation on social networks with
  sentiment-based matrix factorization.
\newblock \emph{Future Generation Computer Systems}, 78:\penalty0 430--439,
  2018.

\bibitem[Han et~al.(2024)Han, Shi, and Zhang]{han2024guaranteed}
Rungang Han, Pixu Shi, and Anru~R Zhang.
\newblock Guaranteed functional tensor singular value decomposition.
\newblock \emph{Journal of the American Statistical Association}, 119\penalty0
  (546):\penalty0 995--1007, 2024.

\bibitem[Kadkhodaie and Montanari(2015)]{kadkhodaie2015accelerated}
Mahdi Kadkhodaie and Andrea Montanari.
\newblock Accelerated alternating direction method of multipliers.
\newblock \emph{arXiv preprint arXiv:1505.01883}, 2015.

\bibitem[Kang et~al.(2016)Kang, Peng, and Cheng]{kang2016top}
Zhao Kang, Chong Peng, and Qiang Cheng.
\newblock Top-n recommender system via matrix completion.
\newblock In \emph{Proceedings of the AAAI conference on artificial
  intelligence}, volume~30, 2016.

\bibitem[Klopp(2014)]{klopp}
Olga Klopp.
\newblock Noisy low-rank matrix completion with general sampling distribution.
\newblock \emph{Bernoulli}, 2014.

\bibitem[Kolda and Bader(2009)]{kolda2009tensor}
Tamara~G. Kolda and Brett~W. Bader.
\newblock Tensor decompositions and applications.
\newblock \emph{SIAM Review}, 51\penalty0 (3):\penalty0 455--500, 2009.

\bibitem[Koltchinskii(2011)]{koltchinskii2011oracle}
Vladimir Koltchinskii.
\newblock \emph{Oracle inequalities in empirical risk minimization and sparse
  recovery problems: Ecole D’Et{\'e} de Probabilit{\'e}s de Saint-Flour
  XXXVIII-2008}, volume 2033.
\newblock Springer, 2011.

\bibitem[Koltchinskii et~al.(2011)Koltchinskii, Lounici, and
  Tsybakov]{koltchinskii2011nuclear}
Vladimir Koltchinskii, Karim Lounici, and Alexandre~B. Tsybakov.
\newblock Nuclear-norm penalization and optimal rates for noisy low-rank matrix
  completion.
\newblock \emph{The Annals of Statistics}, 39\penalty0 (5):\penalty0
  2302--2329, 2011.

\bibitem[Larsen et~al.(2024)Larsen, Kolda, Zhang, and
  Williams]{larsen2024tensor}
Brett~W Larsen, Tamara~G Kolda, Anru~R Zhang, and Alex~H Williams.
\newblock Tensor decomposition meets rkhs: Efficient algorithms for smooth and
  misaligned data.
\newblock \emph{arXiv preprint arXiv:2408.05677}, 2024.

\bibitem[Li et~al.(2024)Li, Wang, Wong, and Chan]{li2024pairwise}
Jiangyuan Li, Jiayi Wang, Raymond~KW Wong, and Kwun Chuen~Gary Chan.
\newblock A pairwise pseudo-likelihood approach for matrix completion with
  informative missingness.
\newblock \emph{Advances in Neural Information Processing Systems},
  37:\penalty0 10735--10769, 2024.

\bibitem[Liu et~al.(2013)Liu, Musialski, Wonka, and Ye]{liu2013tensor}
Ji~Liu, Przemyslaw Musialski, Peter Wonka, and Jieping Ye.
\newblock Tensor completion for estimating missing values in visual data.
\newblock \emph{IEEE Transactions on Pattern Analysis and Machine
  Intelligence}, 35\penalty0 (1):\penalty0 208--220, 2013.

\bibitem[Mazumder et~al.(2010)Mazumder, Hastie, and
  Tibshirani]{mazumder2010spectral}
Rahul Mazumder, Trevor Hastie, and Robert Tibshirani.
\newblock Spectral regularization algorithms for learning large incomplete
  matrices.
\newblock \emph{The Journal of Machine Learning Research}, 11:\penalty0
  2287--2322, 2010.

\bibitem[Mu et~al.(2014)Mu, Huang, Wright, and Goldfarb]{mu2014square}
Cun Mu, Bo~Huang, John Wright, and Donald Goldfarb.
\newblock Square deal: Lower bounds and improved relaxations for tensor
  recovery.
\newblock In \emph{International Conference on Machine Learning}, pages 73--81,
  2014.

\bibitem[Muandet et~al.(2017)Muandet, Fukumizu, Sriperumbudur, and
  Schölkopf]{muandet2017kernel}
Krikamol Muandet, Kenji Fukumizu, Bharath~K. Sriperumbudur, and Bernhard
  Schölkopf.
\newblock Kernel mean embedding of distributions: A review and beyond.
\newblock \emph{Foundations and Trends in Machine Learning}, 10\penalty0
  (1-2):\penalty0 1--141, 2017.

\bibitem[Negahban and Wainwright(2012)]{negahban2012restricted}
Sahand~N. Negahban and Martin~J. Wainwright.
\newblock Restricted strong convexity and weighted matrix completion: Optimal
  bounds with noise.
\newblock \emph{Journal of Machine Learning Research}, 13:\penalty0 1665--1697,
  2012.

\bibitem[Rahimi and Recht(2007)]{rahimi2007random}
Ali Rahimi and Benjamin Recht.
\newblock Random features for large-scale kernel machines.
\newblock In \emph{Advances in Neural Information Processing Systems}, 2007.

\bibitem[Smola et~al.(2007)Smola, Gretton, Song, and
  Sch{\"o}lkopf]{smola2007hilbert}
Alex~J. Smola, Arthur Gretton, Le~Song, and Bernhard Sch{\"o}lkopf.
\newblock A hilbert space embedding for distributions.
\newblock \emph{Journal of Machine Learning Research}, 7:\penalty0 1449--1472,
  2007.

\bibitem[Sriperumbudur et~al.(2010)Sriperumbudur, Gretton, Fukumizu,
  Sch{\"o}lkopf, and Lanckriet]{sriperumbudur2010hilbert}
Bharath~K. Sriperumbudur, Arthur Gretton, Kenji Fukumizu, Bernhard
  Sch{\"o}lkopf, and Gert R.~G. Lanckriet.
\newblock Hilbert space embeddings and metrics on probability measures.
\newblock \emph{Journal of Machine Learning Research}, 11:\penalty0 1517--1561,
  2010.

\bibitem[Tucker(1966)]{tucker1966some}
Ledyard~R. Tucker.
\newblock Some mathematical notes on three-mode factor analysis.
\newblock \emph{Psychometrika}, 31\penalty0 (3):\penalty0 279--311, 1966.

\bibitem[Vershynin(2018)]{vershynin2018high}
Roman Vershynin.
\newblock \emph{High-dimensional probability: An introduction with applications
  in data science}, volume~47.
\newblock Cambridge university press, 2018.

\bibitem[Wang et~al.(2021)Wang, Wong, Mao, and Chan]{wang2021matrix}
Jiayi Wang, Raymond K.~W. Wong, Xiaojun Mao, and Kwun Chuen~Gary Chan.
\newblock Matrix completion with model-free weighting.
\newblock In \emph{International Conference on Machine Learning}, pages
  10927--10936. PMLR, 2021.

\bibitem[Wang et~al.(2022)Wang, Wong, and Zhang]{wang2022low}
Jiayi Wang, Raymond K.~W. Wong, and Xiaoke Zhang.
\newblock Low-rank covariance function estimation for multidimensional
  functional data.
\newblock \emph{Journal of the American statistical Association}, 117\penalty0
  (538):\penalty0 809--822, 2022.

\bibitem[Weng and Wang(2012)]{weng2012low}
Zhiyuan Weng and Xin Wang.
\newblock Low-rank matrix completion for array signal processing.
\newblock In \emph{International Conference on Acoustics, Speech and Signal
  Processing (ICASSP)}, pages 2697--2700. IEEE, 2012.

\bibitem[Williams and Seeger(2001)]{williams2001using}
Christopher K.~I. Williams and Matthias Seeger.
\newblock Using the nystr{\"o}m method to speed up kernel machines.
\newblock \emph{Advances in Neural Information Processing Systems}, 2001.

\bibitem[Xu et~al.(2023)Xu, Lin, Luo, and Xu]{xu2023hrst}
Xiuqin Xu, Mingwei Lin, Xin Luo, and Zeshui Xu.
\newblock {HRST-LR}: a hessian regularization spatio-temporal low rank
  algorithm for traffic data imputation.
\newblock \emph{IEEE Transactions on Intelligent Transportation Systems},
  24:\penalty0 11001--11017, 2023.

\bibitem[Yuchi et~al.(2023)Yuchi, Mak, and Xie]{yuchi2023bayesian}
Henry~Shaowu Yuchi, Simon Mak, and Yao Xie.
\newblock Bayesian uncertainty quantification for low-rank matrix completion.
\newblock \emph{Bayesian Analysis}, 18\penalty0 (2):\penalty0 491--518, 2023.

\bibitem[Zhang et~al.(2021)Zhang, Lu, and Jin]{zhang2021artificial}
Qian Zhang, Jie Lu, and Yaochu Jin.
\newblock Artificial intelligence in recommender systems.
\newblock \emph{Complex \& Intelligent Systems}, 7\penalty0 (1):\penalty0
  439--457, 2021.

\bibitem[Zhang and Zhang(2020)]{zhang2020low}
Shuimei Zhang and Yimin~D Zhang.
\newblock Low-rank hankel matrix completion for robust time-frequency analysis.
\newblock \emph{IEEE Transactions on Signal Processing}, 68:\penalty0
  6171--6186, 2020.

\bibitem[Zhao et~al.(2025)Zhao, Wang, and Lou]{zhao2025noisy}
Kun Zhao, Jiayi Wang, and Yifei Lou.
\newblock Noisy low-rank matrix completion via transformed $\ell_1$
  regularization and its theoretical properties.
\newblock In \emph{Proceedings of the 28th International Conference on
  Artificial Intelligence and Statistics}, Proceedings of Machine Learning
  Research. PMLR, 2025.

\bibitem[Zheng et~al.(2024)Zheng, Lou, Tian, and Wang]{zheng2024scale}
Huiwen Zheng, Yifei Lou, Guoliang Tian, and Chao Wang.
\newblock A scale-invariant relaxation in low-rank tensor recovery with an
  application to tensor completion.
\newblock \emph{SIAM Journal on Imaging Sciences}, 17\penalty0 (1):\penalty0
  756--783, 2024.

\end{thebibliography}

\newpage
\appendix

\section{Additional Information Related to Tensor}
\label{sec:add_tensor}


\begin{definition}[Matricization]
\label{def:matricize}
Let $j \in \{1,\dots,d\}$. The $j$-mode matricization of a tensor
$B \in \mathbb{R}^{q_1 \times q_2 \times \cdots \times q_d}$, denoted by $B_{(j)}$,
rearranges the entries of $B$ into a matrix of size
$q_j \times \bigl(\prod_{i \neq j} q_i\bigr)$.
Specifically, the entry of $B_{(j)}$ indexed by $(l_j, k)$ corresponds to the tensor element
$B_{l_1,\ldots,l_d}$, where the column index $k$ is determined by
\(
k = 1 + \sum_{i=1,\, i \neq j}^{d} (l_i - 1)
\Bigl(\prod_{m=1,\, m \neq j}^{i-1} q_m\Bigr).
\)
\footnote{All empty products are defined to be $1$. For example,
$\prod_{m=i}^{j} q_m = 1$ when $i > j$.}
\end{definition}


\begin{definition}[$j$-mode product]
\label{def:n-mode}
Let $j \in \{1,\dots,d\}$. For an array
$B \in \mathbb{R}^{q_1 \times q_2 \times \cdots \times q_d}$
and a matrix $P \in \mathbb{R}^{p_j \times q_j}$,
the $j$-mode product of $B$ and $P$, denoted by $B \times_j P$,
is an array of dimension
\[
q_1 \times \cdots \times q_{j-1} \times p_j \times q_{j+1} \times \cdots \times q_d.
\]
The $(l_1,\dots,l_{j-1}, k, l_{j+1}, \dots, l_d)$-th entry of $B \times_j P$
is defined as
\[
(B \times_j P)_{l_1,\dots,l_{j-1},k,l_{j+1},\dots,l_d}
=
\sum_{i=1}^{q_j}
B_{l_1,\dots,l_{j-1}, i, l_{j+1},\dots,l_d}\, P_{k,i}.
\]
\end{definition}

Figure \ref{tucker_figure} provides a pictorial illustration of a Tucker decomposition.

\begin{figure}[ht]
\centering
\scalebox{0.8}{
\begin{tikzpicture}[baseline=10ex,scale=3,
  tensor/.style={draw=blue!70!black, fill=blue!10, very thick, rounded corners=1pt},
  core/.style={draw=orange!80!black, fill=orange!15, very thick, rounded corners=1pt},
  uone/.style={draw=green!60!black, fill=green!12, very thick, rounded corners=1pt},
  utwo/.style={draw=purple!65!black, fill=purple!12, very thick, rounded corners=1pt},
  uthree/.style={draw=red!65!black, fill=red!12, very thick, rounded corners=1pt},
  hidden/.style={dashed, draw=gray!70, line width=0.8pt},
  lab/.style={font=\small}
]
\coordinate (A1) at (0, 0);
\coordinate (A2) at (0, 1);
\coordinate (A3) at (1, 1);
\coordinate (A4) at (1, 0);
\coordinate (B1) at (0.3, 0.3);
\coordinate (B2) at (0.3, 1.3);
\coordinate (B3) at (1.3, 1.3);
\coordinate (B4) at (1.3, 0.3);

\draw[tensor] (A1) -- (A2) -- (A3) -- (A4) -- cycle;
\draw[tensor] (A2) -- (B2) -- (B3) -- (A3) -- cycle;
\draw[tensor] (A4) -- (B4) -- (B3) -- (A3) -- cycle;

\draw[hidden] (A1) -- (B1);
\draw[hidden] (B1) -- (B2);
\draw[hidden] (B1) -- (B4);
\draw[hidden] (B1) -- (A1);

\draw[tensor] (B2) -- (B3);
\draw[tensor] (A2) -- (B2);
\draw[tensor] (A3) -- (B3);
\draw[tensor] (A4) -- (B4);
\draw[tensor] (B4) -- (B3);

\node[lab] at (0.6,0.6) {$B$};
\node[lab] at (0.6,-0.2) {$(q_1 \times q_2 \times q_3)$};
\end{tikzpicture}

\qquad = \qquad

\begin{tikzpicture}[baseline=9ex,scale=3,
  tensor/.style={draw=blue!70!black, fill=blue!10, very thick, rounded corners=1pt},
  core/.style={draw=orange!80!black, fill=orange!15, very thick, rounded corners=1pt},
  uone/.style={draw=green!60!black, fill=green!12, very thick, rounded corners=1pt},
  utwo/.style={draw=purple!65!black, fill=purple!12, very thick, rounded corners=1pt},
  uthree/.style={draw=red!65!black, fill=red!12, very thick, rounded corners=1pt},
  hidden/.style={dashed, draw=gray!70, line width=0.8pt},
  lab/.style={font=\small}
]

\draw[uone] (0,0) rectangle (0.5,1);
\node[lab] at (0.25,0.5) {$U_1$};
\node[lab] at (0.25,-0.2) {$(q_1 \times r_1)$};


\filldraw[core] (0.7,0.2) rectangle (1.2,0.7);

\filldraw[core] 
  (0.7,0.7) -- (0.9,0.9) -- (1.4,0.9) -- (1.2,0.7) -- cycle;

\filldraw[core]
  (1.2,0.2) -- (1.4,0.4) -- (1.4,0.9) -- (1.2,0.7) -- cycle;

\draw[hidden] (0.7,0.2) -- (0.9,0.4);
\draw[hidden] (0.9,0.4) -- (1.4,0.4);
\draw[hidden] (0.9,0.4) -- (0.9,0.9);

\node[lab] at (1,0.5) {$G$};
\node[lab] at (1.1,0) {$(r_1 \times r_2 \times r_3)$};

\draw[hidden] (0.7,0.2) -- (0.9,0.4);
\draw[core] (1.2,0.2) -- (1.4,0.4);
\draw[core] (1.4,0.4) -- (1.4,0.9);
\draw[core] (1.4,0.9) -- (1.2,0.7);
\draw[core] (1.4,0.9) -- (0.9,0.9);
\draw[core] (0.9,0.9) -- (0.7,0.7);
\draw[hidden] (0.9,0.9) -- (0.9,0.4);
\draw[hidden] (0.9,0.4) -- (1.4,0.4);

\draw[uthree] (0.7,1.1) -- (1.2,1.1) -- (1.6,1.5) -- (1.1,1.5) -- cycle;
\node[lab] at (1.2,1.3) {$U_3$};
\node[lab] at (1.9,1.3) {$(q_3 \times r_3)$};

\draw[utwo] (1.7,0.2) rectangle (2.2,0.8);
\node[lab] at (2.0,0.5) {$U_2$};
\node[lab] at (2.0,0) {$(q_2 \times r_2)$};

\end{tikzpicture}
}
\caption{Tucker decomposition of a third-order array $B \in \mathbb{R}^{q_1\times q_2\times q_3}$. 
    The values in the parentheses are dimensions for the corresponding matrices or tensors.}
\label{tucker_figure}
\end{figure}
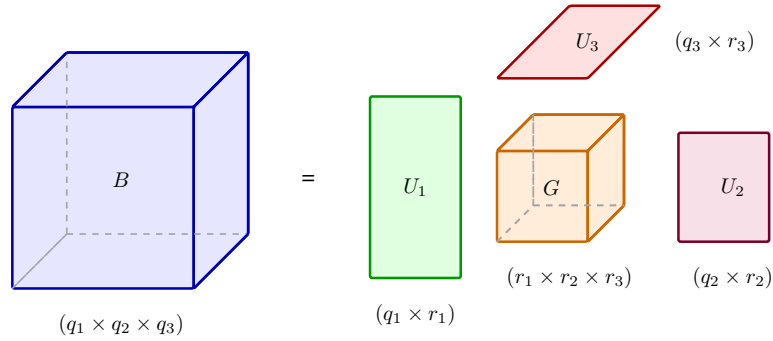

\section{Computational Details}
\subsection{Algorithm}
\label{sec:algo}
In this section, we introduce the accelerated ADMM algorithm for solving \eqref{eqn:repar}. 
We begin with an alternative form of \eqref{eqn:repar}:
\begin{align}\label{eq:obj2}
& \min_{B \in \mathbb{R}^{m_1\times m_2 \times q} } \left\{ \ell(B) 
+ \lambda \left(\sum_{k=1}^3 \beta_j\| D_j\|_*\right) \right\}.\\
& \text{subject to}\  \ 
\calU_k B_k = D_k, \quad k =1,\dots, 3, \quad \sup_{i,j}\|B(i,j,\cdot)\| \leq \alpha,
\end{align}
where $\ell(B) = \sum_{ij} T_{ij} \|B(i,j,\cdot) - Y(i,j,\cdot)\|_2^2$, $D_1 \in \mathbb{R}^{m_1\times (m_2q)}$, $D_2 \in \mathbb{R}^{m_2 \times (m_1q)}$ and $D_3 \in \mathbb{R}^{q \times (m1m2)}$.

Then a standard ADMM algorithm solves the optimization problem \eqref{eq:obj2} by
minimizing the augmented Lagrangian with respect to different variables alternatively.
More explicitly,
at the $(t+1)$-th iteration,
the following updates are implemented:
\begin{subequations}
	\label{eq:upd}
	\begin{align}
	{B} ^{(t+1)}&=\argmin_{\sup_{i,j}\|B(i,j,\cdot)\| \leq \alpha} \left\{{\ell}(B) + \frac{\eta}{2}\sum_{k=1}^{3}\left\|\calU_k B - D_k^{(t)} + V_k^{(t)} \right\|_F^2  \right\}, \label{eq:upd1}\\
	 D_k^{(t+1)}&=\argmin_{ D_k}\left\{\lambda \beta_k\|  D_{k}\|_*+\frac{\eta}{2} \left\| \calU_k B^{(t+1)} - D_{k} +  V_{k}^{(t)} \right\|_F^2\right\} ,\ k=1,2,3,\label{eq:upd3}\\
	 V_k^{(t+1)}&= V_k^{(t)}+ \calU_k B^{(t+1)}- D_k^{(t+1)},\ k=1,2,3,
	\end{align}
\end{subequations}
where $ V_k$ is of the same dimension of $D_k$, for $k=1,2,3$, are scaled Lagrangian multipliers
and $\eta>0$ is an algorithmic parameter.
An adaptive strategy to tune $\eta$ is provided in
\citet{Boyd-Parikh-Chu10}.
Steps \eqref{eq:upd1},  \eqref{eq:upd3}
involve additional optimizations.
and now we discuss how to solve them.

The objective function of \eqref{eq:upd1} is a quadratic function, and so we can easily solve this
with a closed-form solution, and then perform a projection to let it satisfy the sup constraint,  given in Algorithm \ref{admm}.
To solve \eqref{eq:upd3}, we use proximal operator $\mathrm{prox}_v$defined by
\begin{subequations}
	\begin{align}
	\mathrm{prox}_v (A) &= \argmin_{W} \left\{  \frac{1}{2}  \|W -  A\|_F^2 + v \| W \|_*\right\} 
	\end{align}
	\label{eq:prox}
\end{subequations}
for $v\ge 0$, where $A \in \mathbb{R}^{n_1\times n_2}$ denotes a general matrix. 
By Lemma 1 in \cite{mazumder2010spectral}, the solutions to  \eqref{eq:prox} have closed forms.
More specifically, 
write the singular value decomposition of
$A$
as
$U\mathrm{diag}((\tilde{a}_1,\dots,\tilde{a}_{\min\{n_1,n_2\}}))  V^\tp $, then
$\mathrm{prox}_v ( A) =  U \mathrm{diag}( \tilde{ c})  V^\tp$
where  $\tilde{ c} = (( \tilde{a}_1-v)_+, (\tilde{a}_2-v)_+,\dots, (\tilde{ a}_{\min\{n_1,n_2\}}-v)_+)$.

The details of the computational procedure are summarized in Algorithm~\ref{admm}, which presents an accelerated variant of the alternating direction method of multipliers (ADMM) incorporating additional steps to improve convergence speed. We denote by $\mathcal{U}_k^{-1}$ the inverse of $\mathcal{U}_k$. Let $\tilde{T} \in \mathbb{R}^{m_1 \times m_2 \times q}$ denote the indicator tensor, where $\tilde{T}(i,j,k) = 1$ if $T_{ij} = 1$ and $\tilde{T}(i,j,k) = 0$ otherwise. We further define $\tilde{I} \in \mathbb{R}^{m_1 \times m_2 \times q}$ as the identity tensor with all entries equal to one.

\begin{algorithm}[h]
  \caption{Accelerated ADMM for solving \eqref{eqn:repar}}
  \label{admm}
  \begin{algorithmic}
    \STATE {\bfseries Input:} Tensor $Y$, parameters $\lambda>0$, $\beta_1$, $\beta_2$, $\beta_3 \ge 0$, $\eta>0$, $T$, $\epsilon$.
    \STATE {\bfseries Initialize:} ${ V}_{1}^{(0)}, {D}_{1}^{(0)}\in \mathbb{R}^{m_1\times(m_1q)}$, ${ V}_{2}^{(0)}, {D}_{2}^{(0)}\in \mathbb{R}^{m_2\times(m_2q)}$, and ${ V}_{3}^{(0)}, {D}_{3}^{(0)}\in \mathbb{R}^{q\times(m_1m_2)}$. $ B^{(0)}\in \mathbb{R}^{m_1\times m_2 \times q}$.
    \FOR{$t=1$ {\bfseries to} $T$}
   \STATE $$\tilde{B}^{(t)} = \left[ 2Y + \sum_{k=1}^3 \{\eta (\calU_k^{-1}) D_k^{(t-1)} - \eta (\calU_k^{-1}) V_k^{(t-1)}\} \right]/ [ \tilde{T} + 3 * \eta * \tilde{I}],$$
    where the operator $/$ denotes element-wise division between tensors. Then $B^{(t)}$ is constructed as 
   \[
B^{(t)}(i,j,\cdot) =
\begin{cases}
\alpha * \tilde{B}(i,j,\cdot)/\|\tilde{B}(i,j,\cdot)\| , & \text{if } \|\tilde{B}(i,j,\cdot)\| > \alpha, \\
\tilde{B}^{(t)}(i,j,\cdot), & \text{otherwise}.
\end{cases}
\]
$$D_k^{(t)} = \mathrm{prox}_{\lambda \beta_k/\eta}(\calU_k B^{(t)} - V_k^{(t-1)}), \qquad k = 1,2, 3. $$
$$V_k^{(t)}= V_k^{(t-1)}+ \calU_k B^{(t)}- D_k^{(t)},\qquad k=1,2,3.$$
    \IF {$\|B^{(t)} - B^{(t-1)}\|_F^2 / \|B^{(t-1)}\|_F \leq \epsilon$} 
   \STATE {\bfseries Break:} set  $B^{(T)} = B^{(t)}$
    \ENDIF
    \ENDFOR
  \end{algorithmic}
\end{algorithm}

\subsection{KME to Density}
\label{sec:num_inversion}
In this section, we provide a way to numerically transform a KME function $\mu$ to the corresponding density function $\rho$.  
Set a grid over the range of $\calX$ as $X_1,\dots, X_K$. Note that 
\begin{align*}
    \mu(t) = \EE_X K(X, t) = \int_{x \in \calX} K(x, t) \rho(x) \mathrm{d}x \approx
   \frac{1}{K}\sum_{t=1}^K K(X_t, t) \rho(X_t). 
\end{align*}
Then we could solve $\rho(X_t), t = 1,\dots, K$ via the following optimization
\begin{align}
    \min_{\rho(X_t), t = 1,\dots, K} \sum_{t'=1}^T \left(\mu(X_t') -  \frac{1}{K}\sum_{t=1}^K K(X_{t'}, t) \rho(X_{t})\right)^2, \\
    \mathrm{subject\  to} \rho(X_t), t = 1,\dots, K \ge 0.
\end{align}

\section{Proof}
\label{sec:proof}
We introduce some notations for simplicity. We take $\calU_j \mu$ as $\mu_{(j)}$ and $\lambda_j = \lambda \beta_j$, for $j =1,2,3$.

\subsection{Proof of Theorem \ref{thm:representor}}
For any $\mu \in \mathbb{R}^{m_1} \otimes \mathbb{R}^{m_2} \otimes \calH $ , we can decompose it into two orthogonal parts $\mu_1$ and $\mu_2$ such that $\mu_1\in \calK$ and $\mu_2 \in \left( \calK\right)^\bot $.  Note that for any $i,j$, by the reproducing property of kernel $K$ and the definition of $\mu_2$, the loss 
$$T_{ij} \left\|\mu(i,j,\cdot) - \frac{1}{N_{ij}} \sum_{l=1}^{N_{ij}}K(X^{(ij)}_{l}, \cdot)\right\|_\calH^2 =  T_{ij} \left\|\mu_1(i,j,\cdot) - \frac{1}{N_{ij}} \sum_{l=1}^{N_{ij}} K(X^{(ij)}_{l}, \cdot)\right\|_\calH^2.$$
Then it suffices to show that $\|\calU_k \mu\|_* \ge\|\calU_k \mu_1\|_* $ for $k=1,2,3$.  
Let $P_{\calK}$ be the projection operator to space $\calK$ and denote the adjoint operator of $A$ by $A^*$. Denote $\sigma_k(A)$ as $k$-th singular value of the operator $A$. We have
\begin{align*}
\sigma_k(\calU_j \mu_{1}) = \sigma_k(\calU_j \mu P_{\calK}) \leq \sigma_k(\calU_j \mu), \qquad j = 1,2\\
\sigma_k(\calU_3 \mu_{1}) =  \sigma_k(P_{\calK} \calU_3 \mu ) \leq \sigma_k(\calU_3 \mu )
\end{align*}
Therefore,
$\|\calU_k \mu\|_* \ge\|\calU_k \mu_1\|_* $ for $k=1,2,3$.

\subsection{Proof of Theorem \ref{thm:approximately} and \ref{thm:low}}

\begin{proof}[Proof of Theorem \ref{thm:approximately}]
   
From the basic inequality $Q(\hat{\mu}) \leq Q(\mu^*)$, we have
\begin{multline*}
\label{eqn:basic_inequality}
     \frac{1}{m_1 m_2} \sum_{(i,j)} T(i,j) \left\| \hat \mu(i,j,\cdot) - \mu^*(i,j,\cdot)\right\|^2_\calH   +  \lambda_1 \|\hat{\mu}_{(1)}\|_* + \lambda_2 \|\hat{\mu}_{(2)}\|_* +\lambda_3 \|\hat{\mu}_{(3)}\|_* \leq    \\
     +  \lambda_1 \|{\mu}^*_{(1)}\|_* + \lambda_2 \|{\mu}^*_{(2)}\|_* +\lambda_3 \|{\mu}^*_{(3)}\|_* \\
     + \frac{2}{m_1 m_2}  \left|\sum_{(i,j)} T(i,j)\left \langle \hat \mu(i,j,\cdot) - \mu^*(i,j,\cdot), \frac{1}{N_{ij}}\sum_{l=1}^{N_{ij}} K(X^{(ij)}_l,\cdot) - \mu^*(i,j,\cdot) \right \rangle_\calH \right|.
\end{multline*}

Denote $\Delta = \hat{\mu} - \mu^*$. By Lemma \ref{lem:frobenius}, we have
    \begin{align*}
      &\pi_L \frac{1}{m_1m_2} \sum_{i,j} \|\Delta(i,j,\cdot)\|_\calH^2 \lesssim 
        \frac{1}{m_1m_2} \sum_{i,j} T_{i,j} \|\Delta(i,j,\cdot)\|_\calH^2 + (2\alpha)^2 \sqrt{\frac{\pi_U \log d}{\pi^2_Lm_1m_2}} \\
        &+ (2\alpha)\sqrt{\frac{\pi_U\log d}{ m}}\frac{1}{\sqrt{m_1m_2}} \left( \frac{\lambda_1}{\lambda_1 + \lambda_2}\left[\|\hat{\mu}_{(1)}\| + \|{\mu}^*_{(1)}\|\right] + \frac{\lambda_2}{\lambda_1 + \lambda_2}\left[\|\hat{\mu}_{(2)}\| + \|{\mu}^*_{(2)}\|\right]  \right) \\
        &\lesssim \left| \sum_{i,j} \langle  \Delta(i,j,\cdot), T(i,j) \epsilon(i,j,\cdot) \rangle_\calH \right|  +  \lambda_1  \|{\mu}^*_{(1)}\|_*+ \lambda_2 \|{\mu}^*_{(2)}\|_* + \lambda_3 \|{\mu}^*_{(3)}\|_* -\lambda_2\|\hat{\mu}_{(2)}\|_*  - \lambda_1  \|\hat{\mu}_{(1)}\|_* - \lambda_3\|\hat{\mu}_{(3)}\|_* \\
        &+ (2\alpha)^2 \sqrt{\frac{\pi_U \log d}{m_1m_2}} + (2\alpha)\sqrt{\frac{\pi_U\log d}{m}} \frac{1}{\sqrt{m_1m_2}} \left( \frac{\lambda_1}{\lambda_1 + \lambda_2}\left[\|\hat{\mu}_{(1)}\| + \|{\mu}^*_{(1)}\|\right] + \frac{\lambda_2}{\lambda_1 + \lambda_2}\left[\|\hat{\mu}_{(2)}\| + \|{\mu}^*_{(2)}\|\right]  \right). 
    \end{align*}
Apply the bound in Lemma \ref{lem:noise}, the above inequality can be further bounded by
\begin{align*}
    &\lesssim  \Delta_N \sqrt{\frac{\pi_U\log m_1m_2}{m}} \frac{1}{\sqrt{m_1m_2}} \left( \frac{\lambda_1}{\lambda_1 + \lambda_2}\left[\|\hat{\mu}_{(1)}\| + \|{\mu}^*_{(1)}\|\right] + \frac{\lambda_2}{\lambda_1 + \lambda_2}\left[\|\hat{\mu}_{(2)}\| + \|{\mu}^*_{(2)}\|\right]  \right) \\
    & +  \lambda_1  \|{\mu}^*_{(1)}\|_*+ \lambda_2 \|{\mu}^*_{(2)}\|_* + \lambda_3 \|{\mu}^*_{(3)}\|_* -\lambda_2\|\hat{\mu}_{(2)}\|_*  - \lambda_1  \|\hat{\mu}_{(1)}\|_* - \lambda_3\|\hat{\mu}_{(3)}\|_* \\
        &+ (2\alpha)^2 \sqrt{\frac{\pi_U \log d}{m_1m_2}} + (2\alpha)\sqrt{\frac{\pi_U\log d}{m}} \frac{1}{\sqrt{m_1m_2}} \left( \frac{\lambda_1}{\lambda_1 + \lambda_2}\left[\|\hat{\mu}_{(1)}\| + \|{\mu}^*_{(1)}\|\right] + \frac{\lambda_2}{\lambda_1 + \lambda_2}\left[\|\hat{\mu}_{(2)}\| + \|{\mu}^*_{(2)}\|\right]  \right). 
\end{align*}
By taking 
\begin{align*}
    (\lambda_1 + \lambda_2)^{-1}=  \bigO \left( \left[\frac{1}{m_1 m_2}\max\left\{ \left(
  \sqrt{M\pi_U} \Delta_N \log(m_1m_2)
 \right),  (\alpha)\sqrt{{\pi_U M \log d }} \right\} \right]^{-1}\right), 
\end{align*}
we have
\begin{align*}
   \frac{1}{m_1m_2} \sum_{i,j} \|\Delta(i,j,\cdot)\|_\calH^2 \lesssim  (2\alpha)^2 \sqrt{\frac{\pi_U \log d}{\pi_L^2m_1m_2}}  + \frac{1}{\pi_L}\left(   \lambda_1  \|{\mu}^*_{(1)}\|_*+ \lambda_2 \|{\mu}^*_{(2)}\|_* + \lambda_3 \|{\mu}^*_{(3)}\|_* \right). 
\end{align*}


\end{proof}

\begin{proof}[Proof of Theorem \ref{thm:low}]
For a matrix $B \in \mathbb{R}^{k_1\times k_2}$, take $U_B$ and $V_B$ as the left and right singular matrixces of $B$ accordingly. Denote $S_U(B)$ and $S_V(B)$ as the linear spaces spanned by column vectors of $U_B$ and $V_B$ respectively. Denote $S^{\ind}_U(B)$ and $S^{\ind}_V(B)$  as the orthogonal components of $S_U(B)$ and $S_V(B)$ respectively. Define
\begin{align}
        P^{\ind}_{B} B' = \mathbf{{P}_{S^{\ind}_U(B)}} B'   \mathbf {{P}_{S^{\ind}_V(B)}}, \qquad P_{B} B' = B' - P^{\ind}_{B} B',
\end{align}
where $\mathbf{{P}_{S}}$ is the projection matrix on the linear space $S$. And we have the following inequality
    \begin{align*}
        \left\| \hat{\mu}_{(1)}\right\|_* = \left\| P_{\mu^*_{(1)}}(\hat{\mu}_{(1)} - \mu^*_{(1)}) +P^{\ind}_{\mu^*_{(1)}}(\hat{\mu}_{(1)} - \mu^*_{(1)}) + \mu^*_{(1)} \right\|_* \\
        \ge  \left\| P^{\ind}_{\mu^*_{(1)}}(\hat{\mu}_{(1)} - \mu^*_{(1)}) + \mu^*_{(1)} \right\|_*  -  \left\| P_{\mu^*_{(1)}}(\hat{\mu}_{(1)} - \mu^*_{(1)})\right\|_*\\
        =  \left\| P^{\ind}_{\mu^*_{(1)}}(\hat{\mu}_{(1)} - \mu^*_{(1)})  \right\|_*  +  \left\| \mu^*_{(1)} \right\|_*  - \left\| P_{\mu^*_{(1)}}(\hat{\mu}_{(1)} - \mu^*_{(1)})\right\|_*.
    \end{align*}
And 
\begin{align}
\label{eqn:decomp}
    \left\| \mu^*_{(1)} \right\|_*  - \left\| \hat{\mu}_{(1)}\right\|_* \leq \left\| P_{\mu^*_{(1)}}(\hat{\mu}_{(1)} - \mu^*_{(1)})\right\|_* - \left\| P^{\ind}_{\mu^*_{(1)}}(\hat{\mu}_{(1)} - \mu^*_{(1)})  \right\|_*.
\end{align}
    Similar inequalities can be obtained for the second and the third mode unfolding. 

    Next, by Lemma \ref{lem:frobenius}, we have
\begin{align*}
     &  \pi_L \frac{1}{m_1m_2} \sum_{i,j} \|\Delta(i,j,\cdot)\|_\calH^2 \leq   \EE  \frac{1}{m_1m_2} \sum_{i,j} T_{i,j} \|\Delta(i,j,\cdot)\|_\calH^2 \lesssim
        \frac{1}{m_1m_2} \sum_{i,j} T_{i,j} \|\Delta(i,j,\cdot)\|_\calH^2 + (2\alpha)^2 \sqrt{\frac{\pi_U \log d}{\pi^2_Lm_1m_2}} \\
       & + \alpha\sqrt{\frac{\pi_U\log d}{ m}}\frac{1}{\sqrt{m_1m_2}} \left( \frac{\lambda_1}{\lambda_1 + \lambda_2}\left[\left\| P^{\ind}_{\mu^*_{(1)}}(\hat{\mu}_{(1)} - \mu^*_{(1)})  \right\|_* + \left\| P_{\mu^*_{(1)}}(\hat{\mu}_{(1)} - \mu^*_{(1)})\right\|_*\right] \right.\\
     &   \left. + \frac{\lambda_2}{\lambda_1 + \lambda_2}\left[\left\| P^{\ind}_{\mu^*_{(2)}}(\hat{\mu}_{(2)} - \mu^*_{(2)})  \right\|_* + \left\| P_{\mu^*_{(2)}}(\hat{\mu}_{(2)} - \mu^*_{(2)})\right\|_*\right]  \right) \\
     & \text{By the basic inquality \eqref{eqn:basic_inequality}:} \\
      &  \leq \frac{\left| \sum_{i,j} \langle  \Delta(i,j,\cdot), T(i,j) \epsilon(i,j,\cdot) \rangle_\calH \right|}{m_1m_2}  +  \lambda_1  \|{\mu}^*_{(1)}\|_*+ \lambda_2 \|{\mu}^*_{(2)}\|_* + \lambda_3 \|{\mu}^*_{(3)}\|_* -\lambda_2\|\hat{\mu}_{(2)}\|_*  - \lambda_1  \|\hat{\mu}_{(1)}\|_* - \lambda_3\|\hat{\mu}_{(3)}\|_* \\
     &   + (2\alpha)^2 \sqrt{\frac{\pi_U \log d}{m_1m_2}} + (2\alpha)\sqrt{\frac{\pi_U\log d}{ m}}\frac{1}{\sqrt{m_1m_2}} \left( \frac{\lambda_1}{\lambda_1 + \lambda_2}\left[\left\| P^{\ind}_{\mu^*_{(1)}}(\hat{\mu}_{(1)} - \mu^*_{(1)})  \right\|_* + \left\| P_{\mu^*_{(1)}}(\hat{\mu}_{(1)} - \mu^*_{(1)})\right\|_*\right] \right.\\
        &\left. + \frac{\lambda_2}{\lambda_1 + \lambda_2}\left[\left\| P^{\ind}_{\mu^*_{(1)}}(\hat{\mu}_{(2)} - \mu^*_{(2)})  \right\|_* + \left\| P_{\mu^*_{(2)}}(\hat{\mu}_{(2)} - \mu^*_{(2)})\right\|_*\right]  \right) \\
         & \text{By Lemma \ref{lem:noise}:} \\
 & \lesssim \frac{2}{m_1 m_2}  \left(
  \sqrt{M\pi_U} \Delta_N \log(m_1m_2)
 \right) \left(\frac{\lambda_1}{\lambda_1 + \lambda_2}\|\Delta_{(1)}\|_*  + \frac{\lambda_2}{\lambda_1 + \lambda_2 }\|\Delta_{(2)}\|_* \right)\\
       & +  \lambda_1 \left( P_{\mu^*_{(1)}}\|\Delta_{(1)}\|_* -  P^{\ind}_{\mu^*_{(1)}}\|\Delta_{(1)}\|_*\right)+ \lambda_2 \left( P_{\mu^*_{(2)}}\|\Delta_{(2)}\|_* -  P^{\ind}_{\mu^*_{(2)}}\|\Delta_{(2)}\|_*\right) + \lambda_3 \left( P_{\mu^*_{(3)}}\|\Delta_{3)}\|_* -  P^{\ind}_{\mu^*_{(3)}}\|\Delta_{(3)}\|_*\right)\\
       & + (2\alpha)^2 \sqrt{\frac{\pi_U \log d}{m_1m_2}} + (2\alpha)\sqrt{\frac{\pi_U\log d}{ m)}}\frac{1}{\sqrt{m_1m_2}} \left( \frac{\lambda_1}{\lambda_1 + \lambda_2}\left[\left\| P^{\ind}_{\mu^*_{(1)}}(\hat{\mu}_{(1)} - \mu^*_{(1)})  \right\|_* + \left\| P_{\mu^*_{(1)}}(\hat{\mu}_{(1)} - \mu^*_{(1)})\right\|_*\right] \right.\\
       & \left. + \frac{\lambda_2}{\lambda_1 + \lambda_2}\left[\left\| P^{\ind}_{\mu^*_{(1)}}(\hat{\mu}_{(2)} - \mu^*_{(2)})  \right\|_* + \left\| P_{\mu^*_{(2)}}(\hat{\mu}_{(2)} - \mu^*_{(2)})\right\|_*\right]  \right) \\
    \end{align*}
If
\begin{align*}
   (\lambda_1 + \lambda_2)^{-1}=  \bigO \left( \left[\max\left\{ \frac{2}{m_1 m_2}\left(
  \sqrt{M\pi_U} \Delta_N \log(m_1m_2)
 \right),  (2\alpha)\sqrt{{\pi_U M \log d M}} \right\} \right]^{-1}\right), 
\end{align*}
Then we have
\begin{align*}
     & \pi_L \frac{1}{m_1m_2} \|\Delta\|^2_F\lesssim   \min\left\{\lambda_1 \left( P_{\mu^*_{(1)}}\|\Delta_{(1)}\|_* \right)+ \lambda_2 \left( P_{\mu^*_{(2)}}\|\Delta_{(2)}\|_* \right) , \lambda_3 \left( P_{\mu^*_{(3)}}\|\Delta_{3)}\|_* \right) \right\} + (2\alpha)^2 \sqrt{\frac{\pi_U \log d}{m_1m_2}}  \\
     & \lesssim  \min\{\lambda_1 \sqrt{r_1} \|\Delta\|_F + \lambda_2 \sqrt{r_2} \|\Delta\|_F + \lambda_3 \sqrt{r_3} \|\Delta\|_F \} +  (2\alpha)^2 \sqrt{\frac{\pi_U \log d}{m_1m_2}} \\
      & \frac{1}{m_1m_2} \|\Delta\|^2_F \lesssim  \min\left\{\frac{\lambda_1^2m_1m_2r_1}{\pi_L^2} + \frac{\lambda_2^2m_1m_2r_2}{\pi_L^2} , \frac{\lambda_3^2m_1m_2r_3}{\pi_L^2} \right\} + (2\alpha)^2 \sqrt{\frac{\pi_U \log d}{\pi_L^2m_1m_2}} 
\end{align*}
\end{proof}

\subsection{Auxiliary Lemmas}
\begin{lemma}
\label{lem:noise}
Take $N = \min_{(i,j):T(i,j)=1}N_{ij}$ and $\epsilon(i,j, \cdot) = \frac{1}{N_{ij}}\sum_{l=1}^{N_{ij}} K(X^{(ij)}_l,\cdot) - \mu^*(i,j,\cdot)$. We have, for any $A \in \calH^{m_1\times m_2}$, with probability at least $1-2/d - 2\exp(-m_1m_2\pi_L)$,
\begin{align*}
& \left| \sum_{i,j} \langle  A(i,j,\cdot), T(i,j) \epsilon(i,j,\cdot) \rangle_\calH \right|  \\
 \lesssim & \min\left\{ \sqrt{M\pi_U} \Delta_N \log(m_1) \left\|A_{(1)}\right\|_*, \sqrt{M\pi_U} \Delta_N \log(m_2) \left\|A_{(2)}\right\|_*, \sqrt{m_1m_2\pi_U} \Delta_N \left\|A_{(3)}\right\|_*\right\}
\end{align*}
where $\Delta_N = \sqrt{\frac{\log m_1m_2 + \log d}{N}}$.
\end{lemma}

\begin{proof}[Proof of Lemma \ref{lem:noise}]
First, conditioned on the event $\sum_{i,j}T(i,j) = n$. 
   From Theorem 3.4 in \cite{muandet2017kernel}, we have with probability at least $1-\delta$, 
   \begin{align*}
\|\epsilon(i,j,\cdot)\|_{\mathcal {H}}
   \leq {\frac {2}{N_{ij}}}C_K+{\sqrt {\frac {\log(2/\delta )}{2N_{ij}}}}, 
\end{align*}
 where $C_K = \sup_{x \in \mathcal{X}} |K(x,x)|$.   By applying a union bound, we have, with probability at least $1- 1/d$,
 \begin{align*}
     \sup_{(i,j):T(i,j)=1} \|\epsilon(i,j,\cdot)\|_{\mathcal {H}} \lesssim \sqrt{\frac{\log n + \log d}{N}} \lesssim \sqrt{\frac{\log m_1m_2+ \log d}{N}}.
 \end{align*}
By marginalizing $n$, we have the above bound hold without conditioning. 
Next, define the event
\begin{align}
    \mathcal{E}: \left\{ \sup_{(i,j):T(i,j)=1} \|\epsilon(i,j,\cdot)\|_{\mathcal {H}} \lesssim \sqrt{\frac{\log n + \log d}{N}} \lesssim \sqrt{\frac{\log m_1m_2+ \log d}{N}}\right\}. 
\end{align}

 Note that
\begin{align*}
    \left| \sum_{i,j} \langle  A(i,j,\cdot), T(i,j) \epsilon(i,j,\cdot) \rangle_\calH \right|   = \left\langle A_{(1)}, T\circ_1 \epsilon_{(1)} \right\rangle \leq \|A_{(1)}\|_* \|T \circ_1 \epsilon_{(1)}\|,
\end{align*}

It remains to bound $\|T \circ_1 \epsilon_{(1)}\|$, where $T$ and $\epsilon$ are independent.

Take $\Delta_N = \sup_{i,j} \|\epsilon(i,j,\cdot)\|_\calH$. 
\begin{align*}
    \|T \circ_1 \epsilon_{(1)}\|^2 & \leq \sup_{\|a\|_2 \leq 1} \sum_{j=1}^{m_2} \left\| \sum_{i=1}^{m_1} a_i T_{ij} \epsilon(i,j,\cdot) \right\|_{\calH}^2 \\
     & \leq \sup_{\|a\|_2 \leq 1} \sum_{j=1}^{m_2} \left\{ \sum_{i=1}^{m_1} a^2_i T^2_{ij} \|\epsilon(i,j,\cdot)\|_\calH^2 + \sum_{i\neq i'} a_i a_{i'} T_{ij} T_{i'j} \langle \epsilon(i,j,\cdot), \epsilon(i',j,\cdot) \rangle_\calH \right\}\\
    & \leq \|M\| \\
    & \leq \|D\| + \|M'\|,
\end{align*}
where $M, D, M' \in \mathbb{R}^{m_1\times m_1}$ are matrices such that $M(i,i') = \sum_{j=1}^{m_2} T_{ij} T_{i'j} \langle \epsilon(i,j,\cdot),  \epsilon(i',j,\cdot) \rangle$, $D$ is diagonal and $D(i,i) = M(i,i)$, $M' = M - D$.

First, we notice that for the diagonal matrix $D$
\begin{align*}
 \|D\| \leq  \sup_{i} \left( \sum_{j=1}^{m_2} \left\{ T^2_{ij} \|\epsilon(i,j,\cdot)\|_\calH^2 \right\} \right)  
\end{align*}
Conditioned on the event $\calE$, one can show that for every $i$, $\sum_{j=1}^{m_2}T^2_{ij}\|\epsilon(i,j,\cdot)\|_\calH^2 - \EE\{\sum_{j=1}^{m_2}T^2_{ij}\|\epsilon(i,j,\cdot)\|_\calH^2\}$ is a subgaussian variable with variance bounded by $m_2 \Delta_N^4 \pi_U$.  Therefore, with probability at least $1-1/d$, we have
\begin{align}
    \|D\|  \lesssim \sqrt{m_2\Delta_N^4 \pi_U} \log n + m_2 \pi_U \Delta_N^2  \lesssim m_2 \pi_U \Delta_N^2 \log d.
\end{align}
Still, we condition on the event $\calE$, $M'$ is a random matrix where every entry is mean zero and subgaussian. 
For entry $(i,i')$, the variance of $\sum_{j=1}^{m_2} T_{i,j} T_{i'j} \langle \epsilon(i,j,\cdot), \epsilon(i',j,\cdot) \rangle_\calH$ is bounded by $m_2 \pi_U^2 \Delta^4_N$. Therefore, by  Theorem 4.4.5 in \cite{vershynin2018high}, we have with probability at least $1-\exp(-m_1)$,
\begin{align*}
   \|M'\| \lesssim \sqrt{m_2 \pi_U^2 \Delta^4_N}(\sqrt{m_1}) = \sqrt{m_1m_2}\pi_U \Delta_N^2
\end{align*}

Therefore, we obtain with probability at least $1-\kappa/d$ for some universal constant $\kappa>0$, 
\begin{align*}
    \|T \circ_1 \epsilon_{(1)}\| \lesssim \sqrt{M \pi_U} \Delta_N \log d\\
\left| \sum_{i,j} \langle  A(i,j,\cdot), T(i,j) \epsilon(i,j,\cdot) \rangle_\calH \right| \lesssim \sqrt{M \pi_U} \Delta_N \log d\|A_{(1)}\|_*.    
\end{align*}
Similarly, one could prove for the second mode unfolding that with probability at least $1-\kappa/d$ for some universal constant $\kappa>0$, 
\begin{align*}
\left| \sum_{i,j} \langle  A(i,j,\cdot), T(i,j) \epsilon(i,j,\cdot) \rangle_\calH \right| \lesssim \sqrt{M \pi_U} \Delta_N \log d\|A_{(2)}\|_*.    
\end{align*}

As for the third-mode unfolding, 

\begin{align*}
    \left| \sum_{i,j} \langle  A(i,j,\cdot), T(i,j) \epsilon(i,j,\cdot) \rangle_\calH \right|   = \left\langle A_{(3)}, T\circ_3 \epsilon_{(3)} \right\rangle \leq \|A_{(3)}\|_* \|T \circ_3 \epsilon_{(3)}\|
\end{align*}

\begin{align*}
    \|T \circ \epsilon_{(3)}\|^2 & \leq \sup_{f\in \calH, \|f\|_\calH \leq 1} \sum_{i,j} T_{ij}^2 \langle \epsilon(i,j, \cdot), f\rangle_\calH^2\\
    & \leq \sum_{i,j} T_{ij}^2 \|\epsilon(i,j, \cdot)\|_\calH^2 \\
\end{align*}
Applying similar argument as above, we have with probability at least $1-1/d$,
\begin{align}
    \sum_{i,j} T_{ij}^2 \|\epsilon(i,j, \cdot)\|_\calH^2 \lesssim m_1m_2\pi_U\Delta_N^2
\end{align}
and therefore
\begin{align}
     \left| \sum_{i,j} \langle  A(i,j,\cdot), T(i,j) \epsilon(i,j,\cdot) \rangle_\calH \right|   \lesssim \sqrt{m_1m_2\pi_U}\Delta_N  \| A_{(3)}\|_*.
\end{align}
\end{proof}

\begin{lemma}
    \label{lem:frobenius}
For any $A \in \calH^{m_1\times m_2}$, under the condition that $\pi_L^{-1} = \smallO(m_1m_2)$, we have with probability at least $1- \kappa/d$ for some constant $\kappa>0$,
 \begin{multline*}
        \frac{1}{m_1m_2}  \sum_{i,j} T_{i,j} \|A(i,j,\cdot)\|_\calH^2 \geq \frac{1}{m_1m_2}  \sum_{i,j} \pi_{i,j} \|A(i,j,\cdot)\|_\calH^2  - \sup_{i,j} \|A(i,j,\cdot)\|_\calH^2\sqrt{\frac{\pi_U \log d}{m_1m_2}} \\
        - \sup_{i,j} \|A(i,j,\cdot)\|_\calH \frac{\|A_{(1)}\|_*}{\sqrt{m_1m_2}} \sqrt{\frac{\pi_U\log d}{m}}, \qquad \forall A. 
    \end{multline*}
    and 
     \begin{multline*}
        \frac{1}{m_1m_2}  \sum_{i,j} T_{i,j} \|A(i,j,\cdot)\|_\calH^2 \geq \frac{1}{m_1m_2}  \sum_{i,j} \pi_{i,j} \|A(i,j,\cdot)\|_\calH^2  - \sup_{i,j} \|A(i,j,\cdot)\|_\calH^2\sqrt{\frac{\pi_U \log d}{m_1m_2}} \\
        - \sup_{i,j} \|A(i,j,\cdot)\|_\calH \frac{\|A_{(2)}\|_*}{\sqrt{m_1m_2}} \sqrt{\frac{\pi_U\log d}{m}}   \qquad \forall A. 
    \end{multline*}
\end{lemma}
\begin{proof}
    Consider the class $\calK(1, \alpha) = \{A: A \in \calH^{m_1\times m_2}, \sup_{i,j} \|A(i,j,\cdot)\|_\calH \leq 1, \|A_{(1)}\|_* \leq \alpha\}$.
Then $\calK(1, \alpha)$ can be represented as 
\begin{align*}
    \left\{ \alpha \sum_{l=1}^{m_1}w_l a^{(l)} f^{(l)}, \sup_{i,j}\left\|\sum_{l=1}^{m_1}w_la^{(l)}_i f^{(l)}_j \right\|_\calH \leq 1/\alpha, \sum_{l=1}^{m_1} w_l \leq 1, w_l \ge 0, \|a^{(l)}\|_2 \leq 1, \sum_{j=1}^{m_2}\|f^{(l)}_j\|^2_\calH \leq 1 \right\}
\end{align*}

Take $\epsilon_{i,j}$ as independent Rademacher random variables.
\begin{align}
    &  \sup_{A_{(1)} \in \calK(1,\alpha)} \left| \sum_{i,j} T_{i,j} \epsilon_{i,j}\|A(i,j,\cdot)\|_\calH^2\right| \nonumber \\
    \leq &  \sup_{w, a, f} \alpha \left| \sum_{i,j} T_{i,j} \epsilon_{i,j}\left\langle\sum_{l=1}^{m_1} w_l a^{(l)}_{i}f^{(l)}_{j}, \sum_{l=1}^{m_1} w_l a^{(l)}_{i}f^{(l)}_{j}\right\rangle_\calH\right|
    \nonumber \\
    \leq &  \alpha  \sup_{w} \left\{ \left| \sum_{l=1}^{m_1} w_l^2 \sup_{\|a^{(l)}\|_2\leq 1, \|f^{(l)}\|_{\calH^{m_2}} \leq 1} \sum_{i,j} T_{i,j} \epsilon_{i,j} {a^{(l)}_i}^2 \|f^{(l)}_j\|_\calH^2\right|  \right. \nonumber \\
    & \left. +  2\sum_{l\neq l} w_l w_{l'} \sup_{\|a^{(l)}\|_2\leq 1, \|f^{(l)}\|_{\calH^{m_2}}}\left| \sum_{i,j} T_{i,j} \epsilon_{i,j} a^{(l)}_i a^{(l')}_i \left\langle f^{(l)}_{j}, f^{(l')}_{j}\right\rangle_\calH\right| \right\}\nonumber\\
    \leq &  \alpha  \sup_{w} \left\{ \left|\sum_{l=1}^{m_1} w_l^2 \sup_{\|a^{(l)}\|_2\leq 1, \|f^{(l)}\|_{\calH^{m_2}} \leq 1}  \sum_{i,j} T_{i,j} \epsilon_{i,j} {a^{(l)}_i}^2 \|f^{(l)}_j\|_\calH^2\right| + 2\sum_{l\neq l} w_l w_{l'} \|T\circ \epsilon\|  \right\} \label{eqn:temp1}
\end{align}
The second last inequality is due to $(\sum_{i=1}^{m_1} [a^{(l)}_i a^{(l')}_i]^2 \leq (\sum_{i=1}^{m_1} [a^{(l)}_i]^2) (\sum_{i=1}^{m_1} [a^{(l')}_i]^2) \leq 1$ and $(\sum_{i=1}^{m_2} \langle f_j^{(l)}, f_j^{(l')}\rangle_\calH ^2 \leq (\sum_{i=1}^{m_2} \|f^{(l)}_j\|_\calH^2) (\sum_{i=1}^{m_2} \|f^{(l')}_j\|_\calH^2) \leq 1$.

Note that for any $i,j$,
\begin{align*}
\left\|\sum_{l=1}^{m_1}w_la^{(l)}_i f^{(l)}_j \right\|^2_\calH &\leq (1/\alpha)^2\\ 
\sum_{l=1}^{m_1}w_l^2 [a_i^{(l)}]^2 \|f_j^{(l)}\|_\calH^2 &\leq (1/\alpha)^2 +  \left| \sum_{l\neq l} w_l w_{l'} a^{(l)}_i a^{(l')}_i \left\langle f^{(l)}_{j}, f^{(l')}_{j}\right\rangle_\calH\right| \\
&\leq (1/\alpha)^2 + \left|\sum_{l\neq l} w_l w_{l'} \right| \leq (1/\alpha)^2 + 1 \\
\sup_{l}w_l^2 [a_i^{(l)}]^2 \|f_j^{(l)}\|_\calH^2  &\leq (1/\alpha)^2 + 1 
\end{align*}

From the contraction  inequality for Rademacher complexity, we have
\begin{align*}
     & \sum_{l=1}^{m_1} \mathbb{E} \sup_{a, f, w} 
     \left| \sum_{i,j} T_{i,j} \epsilon_{i,j} 
     w_l^2 [a_i^{(l)}]^2 \|f_j^{(l)}\|_\calH^2 \right|\\
     \leq & 2 \sqrt{(1/\alpha)^2 + 1}  \sum_{l=1}^{m_1}  \mathbb{E} \sup_{a, f, w}  w_l \left| \sum_{i,j} T_{i,j} \epsilon_{i,j} 
     [a_i^{(l)}] \|f_j^{(l)}\|_\calH \right|\\
     \leq & 2 \sqrt{(1/\alpha)^2 + 1} \sup_{w} \sum_{l} w_l \mathbb{E} \sup_{\|a\|_2\leq 1, \sum_{j=1}^{m_2}\|f_j\|^2_\calH \leq 1} \left|  \sum_{i,j} T_{i,j} \epsilon_{i,j} a_i \|f_j\|_\calH \right| \\
     \leq & 2 \sqrt{(1/\alpha)^2 + 1}  \sup_{w} \sum_{l} w_l 
 \mathbb{E}  \|T\circ \epsilon\| \leq 2 \sqrt{(1/\alpha)^2 + 1}  \sup_{w} \mathbb{E}  \|T\circ \epsilon\| 
\end{align*}

Therefore, due to the symmetrization inequality, we have
\begin{align*}
& \mathbb{E} \sup_{A \in \calK(1,\alpha)} \left| \sum_{i,j} T_{i,j} \|A(i,j,\cdot)\|_\calH^2 - \sum_{i,j} \pi_{i,j} \|A(i,j,\cdot)\|_\calH^2\right| \\
   & \leq 2 \mathbb{E} \sup_{A_{(1)} \in \calK(1,\alpha)} \left| \sum_{i,j} T_{i,j} \epsilon_{i,j}\|A(i,j,\cdot)\|_\calH^2\right|\\
   & \lesssim \alpha [\sqrt{(1/\alpha)^2 + 1} + 1] \mathbb{E}  \|T\circ \epsilon\| \\
   & \lesssim \alpha \sqrt{\pi_U M}
\end{align*}
the third inequality is due to \eqref{eqn:temp1} and  the last inequality can be obtained from Lemma S7 in \cite{wang2021matrix} by setting $\hat{W}$ to be the matrix whose entries are all ones.   

Next, we adopt the Talagrand concentration inequality and the peeling argument to establish the concentration inequality. Consider $A \in \calK(1,\alpha)$, 
take $X_{i,j} = (T_{i,j} - \pi_{i,j})\| A(i,j,\cdot)\|_\calH^2$.  It is easy to see that $|X_{i,j}| \leq 1$, 
\begin{align}
    \sup_{A_{(1)} \in \calK(1,\alpha)} \EE \sum_{i,j} X^2_{i,j} + 16 \sup_{i,j} |X_{i,j}| \EE  \sup_{A \in \calK(1,\alpha)}|\sum_{i,j} X_{i,j}| \leq m_1m_2 \pi_U + 16 \alpha \sqrt{\pi_U M}. 
\end{align}
Therefore, we could take $U = 1$ and $V = m_1m_2 \pi_U + 16\alpha \sqrt{\pi_U M}$ in Theorem 2.6 in \cite{koltchinskii2011oracle}, which leads to 
\begin{align}
    \Pr \left(\left[ \sup_{A_{(1)} \in \calK(1,\alpha)} \left|\sum_{i,j} X_{i,j}\right| -  \EE \sup_{A_{(1)} \in \calK(1,\alpha)} \left|\sum_{i,j} X_{i,j} \right|\right] \ge t \right)  \nonumber\\
  \leq K \exp\left\{-\frac{1}{K} \frac{t}{U} \log \left(1 + \frac{tU}{V}\right) \right\}  \leq  K \exp\left\{-\frac{1}{K} \frac{t^2}{m_1m_2 \pi_U + 16 \alpha \sqrt{\pi_U M}+ t} \right\} 
\end{align}
for some universal constant $K>0$. 

When $\alpha \leq \sqrt{m}$, we have 
\begin{align}
     \Pr \left(\sup_{A_{(1)} \in \calK(1,\alpha)} \left|\sum_{i,j} X_{i,j}\right| \ge 2 \sqrt{\pi_U m_1m_2 \log d}  \right)  \nonumber \\
     \leq \Pr \left(\sup_{A_{(1)} \in \calK(1,\sqrt{m})} \left|\sum_{i,j} X_{i,j}\right| \ge 2 \sqrt{\pi_Um_1m_2\log d} \right)
    \leq  \frac{c}{d}  \label{eqn:bound1}
\end{align}
for some constant $c>0$.
When $\alpha \ge \sqrt{m}$, we have for $t'\ge 1/2$,
\begin{align}
     \Pr \left(\sup_{A_{(1)} \in \calK(1,\alpha)} \left|\sum_{i,j} X_{i,j}\right| \ge t' \alpha \sqrt{M \pi_U \log d} \right) \leq  K' \exp\left\{-\frac{1}{K'} \frac{{t'}^2 \alpha^2 M \pi_U \log d }{m_1m_2\pi_U + t' \alpha \sqrt{M \pi_U \log d}} \right\},
\end{align}
for some universal constant $K'>0$.
Next, consider a sequence of the sets
\begin{align}
    \calS_l := \{A \in \calK(1,\alpha): 2^{l-1}\sqrt{m} \leq \alpha \leq  2^l \sqrt{m}\}, \quad l = 1,2,\dots
\end{align}
Then we have
\begin{align}
    &\Pr\left(\sup_{\sup_{i,j}\|A(i,j,\cdot)\|_\calH \leq 1}  \frac{|\sum_{i,j} X_{i,j}|}{\|A_{(1)}\|_*} \ge \sqrt{M \pi_U \log d}\right) \nonumber\\
   & \leq \sum_{l=1}^{\infty}  \Pr\left(\sup_{A \in \calS_l}  \frac{|\sum_{i,j} X_{i,j}|}{\|A_{(1)}\|_*} \ge \sqrt{M \pi_U \log d}\right) \nonumber\\
   & \leq \sum_{l=1}^{\infty} \Pr\left( \sup_{A \in \calK(1,2^l \sqrt{m})} |\sum_{i,j} X_{i,j}| \ge 2^{l-1}\sqrt{m} \sqrt{M \pi_U \log d}\right) \nonumber \\
   & \leq \sum_{l=1}^{\infty}  K' \exp\left\{-\frac{1}{K'} \frac{ 2^{2l-2} \log d}{1 +  2^{l-1} \sqrt{\frac{\log d}{m_1m_2 \pi_U}}} \right\} \nonumber \\
   & \leq \sum_{l=1}^{\infty}  K' \exp\left\{-\frac{1}{K'} \frac{ 2^{l-1} \log d}{1 +  \sqrt{\frac{\log d}{m_1m_2 \pi_U}}} \right\}\nonumber \\
  & \leq  \sum_{l=1}^{\infty}  K' \exp\left\{-\frac{\log 2}{2K'} \frac{ l \log d}{1 +  \sqrt{\frac{\log d}{m_1m_2 \pi_U}}} \right\}\nonumber \\
  & \leq K' \frac{\exp\left\{-\frac{\log 2}{2K'} \frac{\log d}{1 +  \sqrt{\frac{\log d}{m_1m_2 \pi_U}}} \right\}}{1- \exp\left\{-\frac{\log 2}{2K'} \frac{  \log d}{1 +  \sqrt{\frac{\log d}{m_1m_2 \pi_U}}} \right\}} \leq \frac{c}{d}, \label{eqn: peeling}
\end{align}
for some constant $c>0$.

Combine \eqref{eqn:bound1} and \eqref{eqn: peeling}, we have 
with probability at least $1-\kappa/d$ for some constant $\kappa>0$ and for any $A$ such that $\sup_{i,j} \|A(i,j,\cdot)\|_\calH \leq 1$,
\begin{align}
    \left|  \frac{1}{m_1m_2}  \sum_{i,j} T_{i,j} \|A(i,j,\cdot)\|_\calH^2 -  \frac{1}{m_1m_2}  \sum_{i,j} \pi_{i,j} \|A(i,j,\cdot)\|_\calH^2\right| \leq \sqrt{\frac{\pi_U \log d}{m_1m_2}} + \|A_{(1)}\|_* \frac{\sqrt{M \pi_U \log d}}{m_1m_2}.
\end{align}
Therefore, 
    \begin{multline*}
        \frac{1}{m_1m_2}  \sum_{i,j} T_{i,j} \|A(i,j,\cdot)\|_\calH^2 \geq \frac{1}{m_1m_2}  \sum_{i,j} \pi_{i,j} \|A(i,j,\cdot)\|_\calH^2  - \sup_{i,j} \|A(i,j,\cdot)\|_\calH^2\sqrt{\frac{\pi_U \log d}{m_1m_2}} \\
        - \sup_{i,j} \|A(i,j,\cdot)\|_\calH \frac{\|A_{(1)}\|_*}{\sqrt{m_1m_2}} \sqrt{\frac{\pi_U\log d}{m}}   
    \end{multline*}
    Similar results can be obtained for the second-mode unfolding. 

\end{proof}

\end{document}